\documentclass[11pt]{article}

\usepackage{amsmath, amsfonts, amssymb, amsthm,  graphicx, mathtools, enumerate}

\usepackage[blocks, affil-it]{authblk}

\usepackage[numbers, square]{natbib}
\usepackage[CJKbookmarks=true,
            bookmarksnumbered=true,
			bookmarksopen=true,
			colorlinks=true,
			citecolor=red,
			linkcolor=blue,
			anchorcolor=red,
			urlcolor=blue]{hyperref}
\usepackage[usenames]{color}

\usepackage[letterpaper, left=1.2truein, right=1.2truein, top = 1.2truein, bottom = 1.2truein]{geometry}
\usepackage[ruled, vlined, lined, commentsnumbered]{algorithm2e}

\usepackage{prettyref,soul}

\newtheorem{lemma}{Lemma}[section]

\newtheorem{thm}{Theorem}[section]

\newtheorem{corollary}{Corollary}[section]

\newrefformat{eq}{(\ref{#1})}
\newrefformat{chap}{Chapter~\ref{#1}}
\newrefformat{sec}{Section~\ref{#1}}
\newrefformat{algo}{Algorithm~\ref{#1}}
\newrefformat{fig}{Fig.~\ref{#1}}
\newrefformat{tab}{Table~\ref{#1}}
\newrefformat{rmk}{Remark~\ref{#1}}
\newrefformat{clm}{Claim~\ref{#1}}
\newrefformat{def}{Definition~\ref{#1}}
\newrefformat{cor}{Corollary~\ref{#1}}
\newrefformat{lmm}{Lemma~\ref{#1}}
\newrefformat{lemma}{Lemma~\ref{#1}}
\newrefformat{prop}{Proposition~\ref{#1}}
\newrefformat{app}{Appendix~\ref{#1}}
\newrefformat{ex}{Example~\ref{#1}}
\newrefformat{exer}{Exercise~\ref{#1}}
\newrefformat{soln}{Solution~\ref{#1}}
\newrefformat{cond}{Condition~\ref{#1}}

\def\text#1{\mbox{\rm #1}}

\newcommand{\fnorm}[1]{\|#1\|_{\rm F}}

\def\T{{ \mathrm{\scriptscriptstyle T} }}

\title{Posterior Contraction Rates of the Phylogenetic Indian Buffet Processes
}
\author[1]{Mengjie Chen}
\author[2]{Chao Gao}
\author[1]{Hongyu Zhao}
\affil[1]{
University of North Carolina, Chapel Hill
}
\affil[2]{
Yale University
}

\begin{document}
\maketitle

\begin{abstract}
  By expressing prior distributions as general stochastic processes, nonparametric Bayesian methods provide a flexible way to incorporate prior knowledge and constrain the latent structure in statistical inference. The Indian buffet process (IBP) is such an example that can be used to define a prior distribution on infinite binary features, where the exchangeability among subjects is assumed. The phylogenetic Indian buffet process (pIBP), a derivative of IBP, enables the modeling of non-exchangeability among subjects through a stochastic process on a rooted tree, which is similar to that used in phylogenetics, to describe relationships among the subjects. In this paper, we study the theoretical properties of IBP and pIBP under a binary factor model. We establish the posterior contraction rates for both IBP and pIBP and substantiate the theoretical results through simulation studies. This is the first work addressing the frequentist property of the posterior behaviors of IBP and pIBP. We also demonstrated its practical usefulness by applying pIBP prior to a real data example arising in the field of cancer genomics where the exchangeability among subjects is violated.
\smallskip

\textbf{Bayesian Nonparametrics, Indian Buffet Process, Latent Factor Analysis, Cancer Genomics.} 
\end{abstract}


\section{Introduction}
Recently nonparametric Bayesian approaches have become popular methods in machine learning and other fields to learn structural information from data. By expressing prior distributions as general stochastic processes, nonparametric Bayesian methods provide flexible ways to incorporate prior knowledge and constrain the latent structure. The Indian buffet process (IBP) is such a stochastic process that can be used to define a prior distribution where the latent structure is presented in the form of a binary matrix with a finite number of rows and an infinite number of columns \citep{griffiths05,knowles2011nonparametric}. The exchangeability among subjects is assumed in IBP, i.e., the joint probability of the subjects being modeled by the prior is invariant to permutation. In certain applications, exogenous information may suggest certain groupings of the subjects, such as studies involving cancer patients with different subtypes. In these cases, treating all  subjects exchangeable using IBP is not appropriate. As an alternative, the phylogenetic Indian buffet process (pIBP) \citep{miller2012phylogenetic} provides a flexible framework to incorporate prior structural information among  subjects for more accurate statistical inference. In pIBP, the dependency structure among subjects is captured by a stochastic process on a rooted tree similar to that used in phylogenetics. As a derivative of IBP, pIBP inherits many of the nice features of IBP including inducing sparsity and allowance of a potentially infinite number of latent factors. In addition, pIBP provides an effective approach to incorporate useful information on the relationship among subjects without losing computational tractability.    

Despite many successful applications of IBP and its variants in many areas \citep{griffiths2011indian}, as far as we know, there has not been any theoretical investigation of their posterior behaviors. Suppose there is a true data-generating process, do the posterior distributions of IBP and pIBP concentrate on the truth?
In the parametric setting where the number of parameters is fixed, the posterior distribution is well behaved according to the classical Bernstein-von Mises theorem \citep{lecam00}. However, when the prior charges a diverging or an infinite number of parameters, whether the posterior distribution still possesses such convergence properties is no longer guaranteed. IBP prior and pIBP prior belong to the second situation because they are stochastic processes on  infinite binary matrices. Besides the issue of posterior convergence, we are also interested in the question whether the extra information in pIBP prior would lead to better posterior behavior than that of IBP prior.

In this paper, we study the theoretical properties of IBP and pIBP under a binary factor model. Posterior contraction rates are derived for both priors under various settings. By imposing a group structure on the true binary factor matrix, pIBP is proved to have faster convergence rates than IBP whenever the group structure is well-specified by the phylogenetic tree. Even when the group structure is mis-specified by pIBP, it still has the same convergence rate as that of IBP.  To the best of our knowledge, this is the first work addressing the frequentist property of the posterior behaviors of both IBP and pIBP.

We further substantiated the theoretical results through simulation studies. Our simulations show that pIBP is an attractive alternative to IBP when subjects can be related through a tree structure based on some prior information. Moreover, even when the tree structure is mis-specified in pIBP prior, the posterior behavior is still comparable to that of IBP prior, suggesting a robust property of pIBP. We further apply pIBP to analyze cancer genomics data to demonstrate its practical usefulness. 

We organize the rest of the paper as follows. Section \ref{sec:factormodel} introduces a binary factor model, which is the probabilistic setting of the paper. The definitions of IBP and pIBP are reviewed in Section \ref{sec:IBPprior}.  Section \ref{sec:postcontra} presents our theoretical studies of the posterior contraction rates of IBP and pIBP. Simulation studies are carried out in Section \ref{sec:simu}. Sections \ref{sec:real} presents the analysis of a TCGA data set using pIBP. Section \ref{sec:disc} discusses related work on factor models and an extension of our theoretical results. Proofs for theoretical results are collected in the supplementary materials.

\section{Problem Setting} \label{sec:factormodel}

\subsection{Notation}

We denote $\max(a,b)$ by $a\vee b$ and $\min(a,b)$ by $a\wedge b$. For two positive sequences $\{a_n\}$ and $\{b_n\}$, $a_n\lesssim b_n$ means there exists a $C>0$, such that $a_n\leq Cb_n$ for all $n$.  For a matrix $A=(a_{ij})_{m\times n}$, denote its matrix Frobenius norm by $||A||_F=\Big(\sum_{i=1}^m\sum_{j=1}^na_{ij}^2\Big)^{1/2}$. For a set $S$, denote its cardinality by $|S|$. The symbol $\Pi$ stands for the prior probability distribution associated with the mixture of IBP or pIBP defined in Section \ref{sec:mixture}, and $\Pi(\cdot|X)$ is the corresponding posterior distribution.

\subsection{Binary Factor model} 

Let $X=(x_{ij})_{n\times p}$ denote the observed data matrix, where each of the $n$ rows represents one individual and each of the $p$ columns represents one measurement. We hypothesize that the measurement profiles can be characterized by latent factors. We model the effects of these latent factors $Z$ on $X$ through the following model:
$$X=ZA+E,$$
where $Z=(z_{ik})_{n\times K}$ is a binary factor matrix, and $A=(a_{kj})_{K\times p}$ is a loading matrix. The status of $z_{ik}$, which takes a value of $1$ or $0$,  indicates the presence or the absence of the $k$th factor in the $i$th individual. The value of $a_{kj}$ weighs the contribution to the $j$th measurement from the $k$th factor. We assume that each entry of $E=(e_{ij})_{n\times p}$ follows  $N(0,\sigma_X^2)$ independently. Let each entry of $A$ follow  $N(0,\sigma_A^2)$ independently, and $A$ is independent of $E$. Conditioning on $A$, $(X\mid A)$ follows a matrix normal distribution with mean $ZA$. Integrating out $A$ with respect to its distribution, each column of $X$ follows
\begin{equation}
(x_{1j},\ldots,x_{nj})^{\T}\sim N(0,\sigma_A^2ZZ^{\T}+\sigma_X^2I), \label{eq:basic}
\end{equation}
independently for $j=1,\ldots,p$. Formula (\ref{eq:basic}) shows the covariance structure across individuals imposed by the binary factor model. From this representation, it is easy to see that the matrix $ZZ^{\T}$ and the variance components $\sigma_A^2$ and $\sigma_X^2$ uniquely determine the data generating process.

\subsection{Feature Similarity Matrix $ZZ^{\T}$} \label{sec:matrix}

We name $ZZ^{\T}$ the feature similarity matrix because of its important statistical meaning as reflected in (\ref{eq:basic}). An identifiability issue is that the distribution of (\ref{eq:basic}) will not change if one reorder the columns of the factor matrix $Z$. Thus, $Z$ is not identifiable in the model. However, the feature similarity matrix $ZZ^{\T}$, according to (\ref{eq:basic}), is identifiable.
We denote each element of this matrix by $ZZ^{\T}=(\xi_{ij})_{n\times n}$. Each row/column of this matrix $ZZ^{\T}$ describes the feature similarity between a particular individual and the other $n-1$ individuals. Note that
$$\xi_{ij}=\sum_{k=1}^K z_{ik}z_{jk}=|\{k: z_{ik}=z_{jk}=1\}|.$$
Thus, the diagonal element $\xi_{ii}$ denotes the number of factors possessed by the $i$th individual, and the off-diagonal entry $\xi_{ij}$ is the number of the factors shared between the $i$th and $j$th individuals. In short, the feature similarity matrix $ZZ^{\T}$ characterizes the latent feature sharing structure among samples. For the $i$th individual, we define $d_i=\sum_{j\neq i}\xi_{ij}$
as its degree. When we have a group structure among the samples, the individual with the highest degree has the most shared factors among a group. That particular individual is a representative prototype for that group.

\section{Tree Structured Indian Buffet Process Prior} \label{sec:IBPprior}

\subsection{A Bayesian Framework}

To pursue a full Bayesian approach, we put a prior distribution on the triple $(Z,\sigma_A^2,\sigma_X^2)$.   The choice of the prior on $(\sigma_A^2,\sigma_X^2)$ is not essential, because for asymptotic purpose (when $n$ and $p$ are large), the prior effect on the parametric part $(\sigma_A^2,\sigma_X^2)$ is negligible. In contrast, the prior on the  binary matrix $Z$ is important. Since we do not specify the number of columns $K$ in advance, the potential number of parameters in $Z$ is infinite. It is well-known that when the number of parameters diverges, Bayesian method is no longer guaranteed to be consistent \citep{diaconis86}. Thus, the choice of the  prior on $Z$ is important.  According to  the model representation (\ref{eq:basic}), the order of the columns of $Z$ is not identifiable. In other words, we cannot tell the first factor from the second. Instead of specifying a prior on $Z$, we specify a prior on the equivalent class $[Z]$, where $[Z]$ denotes the collection of matrices $Z$ which are equivalent by reordering the columns.

We describe two priors on $[Z]$ in this section, the Indian buffet process proposed by \cite{griffiths05}, and its tree-structured generalization, the phylogenetic Indian buffet process proposed by \cite{miller2012phylogenetic}. Both are priors on sparse infinite binary matrices.

\subsection{Indian Buffet Process} \label{sec:IBP}

We describe the Indian buffet process (IBP) on $[Z]$ by its stick-breaking representation derived in \cite{teh07}. Given some $\alpha>0$, first draw $v_k\sim \text{Beta}(\alpha,1)\ (k=1,2,\ldots)$  independently and identically distributed. Then, $p_k$ is 
\begin{equation}
p_k=\prod_{i=1}^k v_i\quad(k=1,2,\ldots). \label{eq:stick}
\end{equation}
Given $\{p_k\}$, $z_{ik}$ is drawn independently from a Bernoulli distribution with parameter $p_k$ for $i=1,\ldots,n$ and $k=1,2,\ldots$. The final matrix $Z$ drawn in this way has dimension $n\times K^+$, where $K^+$ is the number of nonzero columns. According to \cite{griffiths05}, $K^+$ follows a Poisson distribution with mean $\alpha\sum_{k=1}^n k^{-1}$. Thus, it is finite with probability $1$.  The IBP prior on $[Z]$ is the image measure induced by the equivalence map $Z\mapsto [Z]$. A larger $\alpha$ indicates a larger $K^+$ in the prior modeling.

\subsection{Phylogenetic Indian Buffet Process} \label{sec:pIBP}

The phylogenetic Indian buffet process (pIBP) also starts with drawing $\{p_k\}$ as in (\ref{eq:stick}). Different from IBP, given $p_k$, the entries of the $k$th column of $Z$ are not independent in pIBP. Their dependency structure is captured by a stochastic process on a rooted tree similar to the models used in phylogenetics \citep{miller2012phylogenetic}. The $n$ individuals are modeled as leaves of the tree. The total edge length from the root to any leaf is $1$. Conditioning on $p_k$, we describe the generating process of the $k$th column of $Z$. First, assign $0$ to the root of the tree. Along any path from the root to a leaf, let the value of any node change to $1$ along any edge of length $t$ with probability $1-\exp(-\gamma_k t)$, where $\gamma_k=-\log(1-p_k)$. Once the value has changed to $1$ along any path from the root, all leaves below that point are assigned value $1$. pIBP prior is defined to be the image measure on $[Z]$. 

\subsection{A Hyperprior on $\alpha$} \label{sec:mixture}

Both IBP and pIBP are determined by the hyper-parameter $\alpha$, which can be tuned in practice. In this paper, we pursue a full Bayesian approach, and  put a Gamma$(1,1)$ prior on $\alpha$ for both IBP and pIBP. Thus, the final prior on the equivalent class $[Z]$ is a mixture of IBP or pIBP after $\alpha$ is integrated out.

\section{Posterior Contraction Rates of IBP and pIBP} \label{sec:postcontra}

\subsection{Convergence of the Feature Similarity Matrix}

In this section, we establish the posterior convergence of both mixture of IBP and mixture of pIBP and characterize their difference by different convergence rates. Such theoretical comparisons are interesting because IBP can be viewed as a special case of pIBP with a default tree. These results will illustrate the impacts of tree structure imposed by the prior.

We define the triple $(Z_0,\sigma_{A,0}^2,\sigma_{X,0}^2)$ to be the true parameter generating the data matrix $X$, where $Z_0$ is an $n\times K_0$ binary matrix and $K_0$ is the number of factors. For the sake of clearer presentation, we assume $\sigma_{A,0}^2=\sigma_{X,0}^2=1$, so that the only unknown parameter is $Z_0$. Denote the data generating process of (\ref{eq:basic}) by $P_Z$, and let $E_Z$ be the associated expectation (and similarly define $P_{Z_0}$ and $E_{Z_0}$). The generalization to the case where $(\sigma_A^2,\sigma_X^2)$ is unknown is covered in the supplementary materials. Let $\Pi$ be the mixture of IBP or pIBP prior on $[Z]$. Note that the matrix $ZZ^{\T}$ does not depend on the order of columns of $Z$, and thus we have $ZZ^{\T}=[Z][Z]^{\T}$. We consider the posterior convergence in the sense of
\begin{equation}
E_{Z_0}\left[\Pi\Big(||ZZ^{\T}-Z_0Z_0^{\T}||_F^2\leq M\epsilon_{n,p}^2\Big| X\Big)\right]\geq 1-\delta_{n,p}, \label{eq:convergence}
\end{equation}
for some sequences $\epsilon_{n,p},\delta_{n,p}$ and constant $M>0$.  When $\delta_{n,p}\rightarrow 0$, this is called posterior contraction of  feature similarity matrix with rate $\epsilon_{n,p}^2$ under the squared Frobenius loss.  We choose to study the posterior contraction in terms of the feature similarity matrix $ZZ^{\T}$  because of both the identifiability issue and statistical interpretation described in Section \ref{sec:matrix}.

\subsection{A General Method for Discrete Priors}

The theory of Bayesian posterior consistency was first studied by \cite{schwartz65}. She proposed a Kullback-Leibler property of the prior and a testing argument to prove weak consistency in the parametric case. The first nonparametric posterior consistency result was obtained by \cite{barron99}, where the idea of testing on the essential support of the prior is used. Later, the same argument was modified to achieve rate of contraction by \cite{ghosal00}. In the current setting of binary factor model, we propose the following general method to prove posterior rate of contraction for priors supported on a discrete set. 

\begin{thm} \label{thm:general}
For any measurable set $U$, and any testing function $\phi$, we have
\begin{equation}
E_{Z_0}\left[\Pi \Big(U\mid X\Big)\right] \leq E_{Z_0}(\phi) +   \frac{1}{\Pi\Big(||ZZ^{\T}-Z_0Z_0^{\T}||_F^2=0\Big)}\sup_{Z\in U}E_Z(1-\phi). \label{eq:general}
\end{equation}
\end{thm}

The theorem can be viewed as a discrete version of the Schwartz theorem \citep{schwartz65}. We take advantage of the discrete nature of the problem, thus avoiding calculating the prior mass of the Kullback-Leibler neighborhood of $P_{Z_0}$. We specify $U$ to be
$$U=\left\{||ZZ^{\T}-Z_0Z_0^{\T}||_F^2> M\epsilon_{n,p}^2\right\}.$$
Thus, in order to obtain (\ref{eq:convergence}), it is sufficient to upper bound the right hand side of (\ref{eq:general}). This can be done by lower bounding $\Pi\Big(||ZZ^{\T}-Z_0Z_0^{\T}||_F^2=0\Big)$ and constructing a testing function for $H_0:Z=Z_0$ and $H_1:Z\in U$ with appropriate type 1 and type 2 error bounds. The existence of such testing function is guaranteed by the following lemma.

\begin{lemma} \label{lem:test}
For any $\epsilon_{n,p}>0$, there is a testing function $\phi$ such that the testing error $E_{Z_0}(\phi)+\sup_{\{||ZZ^{\T}-Z_0Z_0^{\T}||_F^2>M\epsilon_{n,p}^2\}}E_Z(1-\phi)$ is upper bounded by
$$\exp\left\{-Cp\min\Bigg(\frac{M\epsilon_{n,p}^2}{n^2K_0^2},\frac{\sqrt{M}\epsilon_{n,p}}{nK_0}\Bigg)+2\log n\right\} +\exp\Big(-Cp+2\log n\Big),$$
for some universal constant $C>0$ and $M$ introduced in (\ref{eq:convergence}).
\end{lemma}

Therefore, it is sufficient to lower bound the prior mass $\Pi\Big(||ZZ^{\T}-Z_0Z_0^{\T}||_F^2=0\Big)$ to obtain (\ref{eq:convergence}). 

\subsection{Two-Group Tree and Factor Decomposition}

Before studying the prior mass lower bound of IBP and pIBP, we need to specify a non-exchangeable structure among the subjects. To demonstrate the power of pIBP to model non-exchangeability,
 we study a special but representative tree structure, the two-group tree. Let $n$ individuals be labeled by $\{1,\ldots,n\}$. Without loss of generality, we assume $n$ is even. Let $\{1,\ldots,n\}=S_1\cup S_2$, where $S_1=\{1,\ldots,n/2\}$ and $S_2=\{n/2+1,\ldots,n\}$. The tree induced by the two-group structure $(S_1,S_2)$ has one root, two group nodes and $n$ leaves. The two group nodes are connected with the root by two edges of length $\eta\in (0,1)$. Then, the $i$th group node is connected with each member of $S_i$ by an edge of length $1-\eta$, where $i=1,2$. The parameter $\eta$ is the strength of the group structure imposed by the prior $\Pi$. When $\eta=0$, pIBP reduces to IBP.

\begin{figure}
\centering
\includegraphics[width=4in]{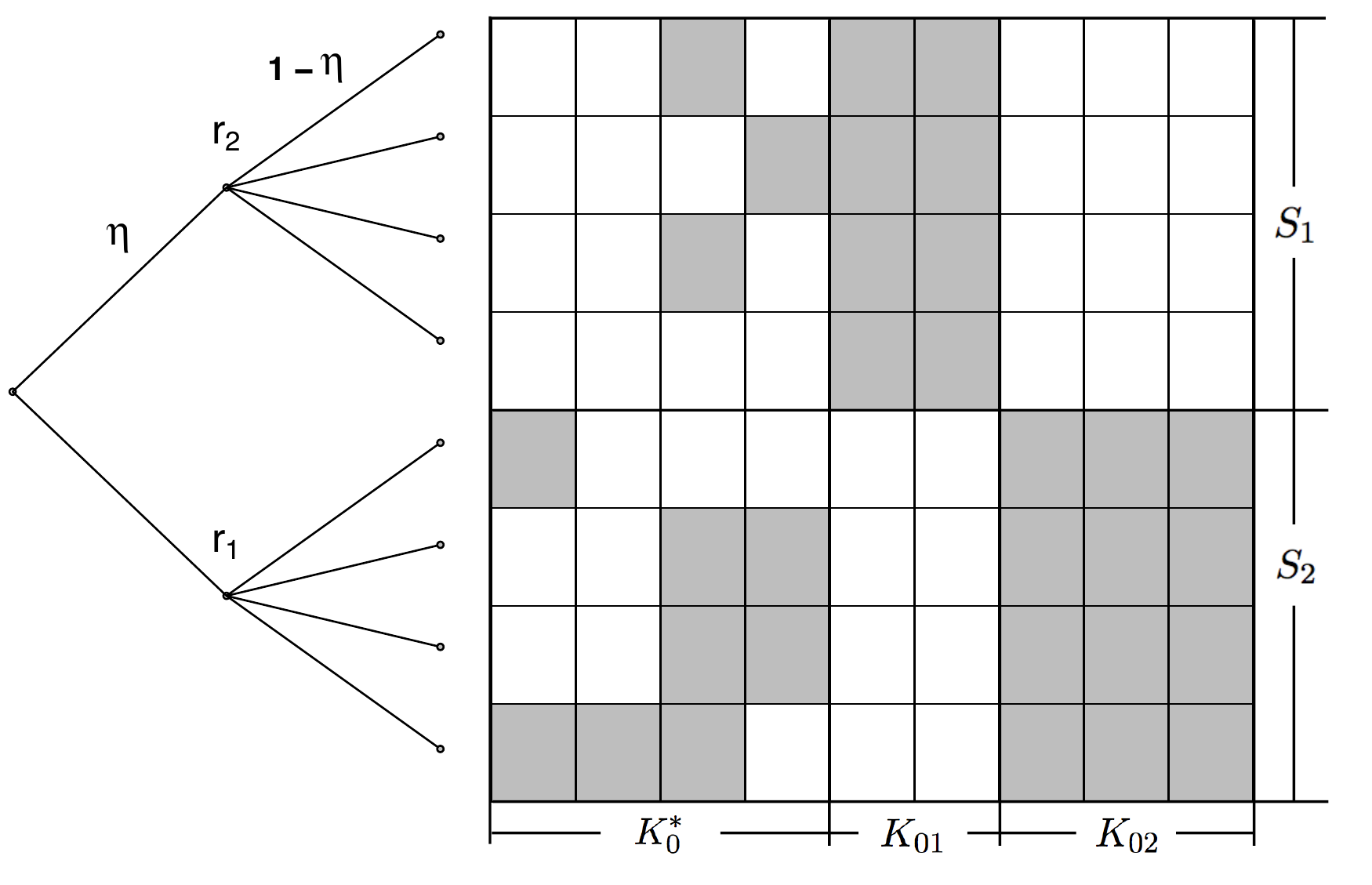}
\caption{An illustration of the two group tree and the factor decomposition.}
\label{fig1}
\end{figure}

Our theory covers three cases. The first case is IBP prior, with no group structure specified in the prior. The second case is the two-group pIBP prior with group structure correctly specified. The third case is the the two-group pIBP prior with group structure mis-specified. Let $Z_0$ have $K_0$ columns, representing $K_0$ factors. Given the two-group structure $(S_1,S_2)$ by the prior $\Pi$, we have the following factor decomposition
\begin{equation}
K_0=K_{01}+K_{02}+K_0^*,\label{eq:decomp}
\end{equation}
where $K_{01}$ is the number of factors unique to $S_1$, $K_{02}$ is the number of factors unique to $S_2$, and $K_0^*$ is the number of factors shared across $S_1$ and $S_2$. Decomposition (\ref{eq:decomp}) is determined by both the structure of $Z_0$ and the prior $\Pi$. It characterizes how well the group structure is specified compared with the true $Z_0$ (see Figure \ref{fig1}). Generally speaking, the smaller  $K_0^*$ is, the better the group structure is specified by $\Pi$.

\subsection{Prior Mass} \label{sec:priorconcen}

Under the two-group structure defined above, we obtain the following prior mass lower bound.
\begin{thm} \label{thm:prior}
For any constant $\eta\in[0,1)$, there exists some constant $C>0$ such that
the prior mass $\Pi\Big(||ZZ^{\T}-Z_0Z_0^{\T}||_F^2=0\Big)$ can be lower bounded by
$$\exp\left(-Cn(({K_0^*+\kappa})^2+1)-Cn\frac{K_0-K_0^*}{(4/3)^{K_0^*+\kappa}}-C(K_0+\kappa)(K_0-K_0^*+1)\right)$$
for any $\kappa\geq 0$.
\end{thm}

Theorem \ref{thm:prior} provides an explicit characterization of the prior mass lower bound as a function of $K_0,K_0^*$. For a larger $K_0^*$, the prior mass will be at a smaller order due to an increased level of misspecification. The prior mass lower bound directly determines the posterior contraction rate according to Theorem \ref{thm:general} and Lemma \ref{lem:test}. In the following, we consider $\eta = 0$ and $\eta \in (0,1)$, separately.

When $\eta=0$, pIBP and IBP are equivalent. The prior does not impose any group structure. 
Thus, in the decomposition (\ref{eq:decomp}), we have $K_0^*=K_0$. By letting $\kappa=0$, Theorem \ref{thm:prior} can be written as
\begin{equation}
\Pi\Big(||ZZ^{\T}-Z_0Z_0^{\T}||_F^2=0\Big)\geq \exp\Big(-C_1nK_0^2\Big). \label{prior:ibp}
\end{equation}
The prior mass lower bound for IBP in (\ref{prior:ibp}) is the benchmark for us to compare IBP with pIBP in various situations.

When $\eta\in (0,1)$, the tree structure plays a role in the prior. In practice, $\eta=1/2$ is often used to characterize moderate group structure belief in the prior \citep{miller2012phylogenetic}. We say the group structure is effectively specified if $K_0^*\lesssim K_0^{1-\beta}$ for some $\beta\in (0,1)$. In this case, the result of Theorem \ref{thm:prior} can be optimized for $k=K_0^*+\kappa$ for any $\kappa\geq 0$. That is, for $n$ sufficiently large $n\gtrsim K_0^{2\beta}$, we have
\begin{equation}
\Pi\Big(||ZZ^{\T}-Z_0Z_0^{\T}||_F^2=0\Big) \geq \exp\Bigg(-C_2n \min_{k\geq K_0^*}\Big(k^2\vee\frac{K_0}{(4/3)^k}\Big)\Bigg),
\end{equation}
which is lower bounded by
$$\exp\Big(-C_2'n K_0^{2(1-\beta)}\Big).$$
This rate is superior to (\ref{prior:ibp}). Thus, pIBP is advantageous over IBP as long as the tree structure captures any group-specific features in the sense that $K_0^*\lesssim K_0^{1-\beta}$.

On the other hand, the group structure is mis-specified if $K_0^*=K_0$. In this case,
we reduce to (\ref{prior:ibp}), so that
$$\Pi\Big(||ZZ^{\T}-Z_0Z_0^{\T}||_F^2=0\Big)\geq \exp\Big(-C_3nK_0^2\Big).$$
Thus, a mis-specified tree structure does not compromise the results, compared to a default tree structure of IBP. One may wonder whether this is due to a possibly loose bound in Theorem \ref{thm:prior}.  By scrutinizing the proof, we found that the slack is at most at a constant level independent of $(n,K_0,K_0^*)$. Thus, the prior mass lower bounds of pIBP with a mis-specified tree and of IBP are essentially the same.

\subsection{Posterior Contraction Rates} \label{sec:postconverge}

Combining Theorem \ref{thm:general}, Lemma \ref{lem:test} and Theorem \ref{thm:prior}, we can derive the posterior contraction rates in the sense of (\ref{eq:convergence}) for both IBP and pIBP.

\begin{thm} \label{thm:postibp}
For the mixture of IBP prior or pIBP prior $\Pi$ on $[Z]$, let $Z_0$ be the true factor matrix. Then, for the binary factor model, there exist $M>0$ and $C'>0$, such that
$$E_{Z_0}\left[\Pi\Bigg(||ZZ^{\T}-Z_0Z_0^{\T}||_F^2\leq M\frac{K_0^4n^3}{p}\Big| X\Bigg)\right] \geq 1-\exp\Big(-C'nK_0^2\Big),$$
as long as $nK_0^2/p=o(1)$.
\end{thm}

\begin{thm} \label{thm:postpibp}
For the mixture of pIBP prior $\Pi$ on $[Z]$ with $\eta\in (0,1)$, let $Z_0$ be the true factor matrix. When $K_0^*\lesssim K_0^{1-\beta}$ and $K^{2\beta}\lesssim n$ for $\beta\in (0,1)$, for the binary factor model, there exist $M>0$ and $C'>0$, such that
$$E_{Z_0}\left[\Pi\Bigg(||ZZ^{\T}-Z_0Z_0^{\T}||_F^2\leq M\frac{K_0^{4-2\beta}n^3}{p}\Big| X\Bigg)\right]\geq 1-\exp\Big(-C'nK_0^{2(1-\beta)}\Big),$$
as long as $nK_0^{2(1-\beta)}/p=o(1)$.
\end{thm}

The above two theorems establish  rates of contraction for the posterior distributions of IBP and pIBP.   The posterior probabilities on the neighborhood of the truth can be arbitrarily close to $1$ in expectation under the true model for sufficiently large $n$, $p$ and $K_0$. The contraction rate is faster for larger $p$ and smaller $n$, because more  variables are helpful to identify the feature similarity of a group of individuals.

Compared with the rate of IBP in Theorem \ref{thm:postibp}, when the tree structure is effectively specified, the upper bound of the rate of pIBP in Theorem \ref{thm:postpibp} is faster by a factor of $K_0^{2\beta}$. Such difference is significant if the number of features $K_0$ is large. Moreover, Theorem \ref{thm:postibp} also suggests that even when the tree structure of pIBP is mis-specified, the rate of contraction is the same as that of IBP, implying the robust property of pIBP. Although our theoretical study is carried out in the simple two-group structure model, similar conclusions can also be obtained under a more complicated structural assumption using the same method.

\section{Simulation Studies} \label{sec:simu}  

In this section, we perform simulations to evaluate the performances of IBP and pIBP. We implemented the Markov chain Monte Carlo algorithm proposed in \cite{miller2012phylogenetic} to perform posterior inference of the feature similarity matrix $ZZ^{\T}$. In the algorithm, the sampling process on the tree structure is expressed as a graphical model, where the prior probabilities can be calculated efficiently by a sum-product algorithm. All the parameters $\sigma_A$, $\sigma_X$, $\alpha$ and $\{p_k\}$ (marginal probabilities of a latent feature equaling $1$) are sampled as part of the overall Markov chain Monte Carlo  procedure.

In the first simulation, we evaluated the performance of IBP, pIBP with a correctly specified tree structure, and pIBP with a mis-specified tree structure (mispIBP). We constructed a set of samples with a clear subgroup structure on $Z_0$. Specifically we simulated data with eight subgroups characterized by six latent factors as illustrated by  Figure \ref{sim1}. Twelve models presented in Table \ref{tab1} are considered. For each model, we generated an $n \times p$ matrix $X=Z_0A+E$ with $(\sigma_{A, 0},\sigma_{X, 0})=(1, $0$\cdot$5$)$. For IBP, we let $\eta=0$ so that pIBP is equivalent to IBP. For pIBP, we let $\eta=\text{0$\cdot$8}$ and a proper tree structure is given. For the mispIBP, we let $\eta=\text{0$\cdot$5}$ and the prior is a mis-specified tree with samples within a subtree assigned to different groups. Estimation error on $Z$ is evaluated in terms of the normalized Frobenius norm of the feature similarity matrix $n^{-1/2}{||ZZ^{\T}-Z_0Z_0^{\T}||_F}$. We further evaluated the latent structure recovery by the number of estimated latent factors. We observed that both IBP and pIBP overestimates the number of latent factors because of the presence of many factors with only a few samples. This is similar to what was proved for Dirichlet and Pitman-Yor processes where the posterior is inconsistent for estimating the number of clusters \citep{miller13}.
Therefore, we reported a truncated estimator of the number of latent factors counting only those factors shared by at least $5$ samples.

The algorithm of \cite{miller2012phylogenetic} is implemented for 1000 MCMC steps. We observe that it guarantees convergence in the problem sizes that are considered in this simulation.

 Generally, the reported twelve models represent two scenarios: the small $p$ scenario and the large $p$ scenario. Remember in our setting, the larger the value of $p$ is, the more accurately we can recover the latent features.
 In the models with a small $p$ ($p=30$ and $20$), the information from data is limited and the inference relies more heavily on the prior information. We found pIBP performs better than the other two methods in both cases. Besides, mispIBP has comparable performance with IBP, implying that pIBP is robust to mis-specified tree structure. The simulation results substantiate the conclusions we have from Theorem \ref{thm:postibp} and Theorem \ref{thm:postpibp}. In the models with large $p$ ($p=100$ and $200$), there is adequate information from the data and the priors play a less important role. Inferences using different priors lead to similar results.

\begin{figure}
\centering
\includegraphics[width=5in]{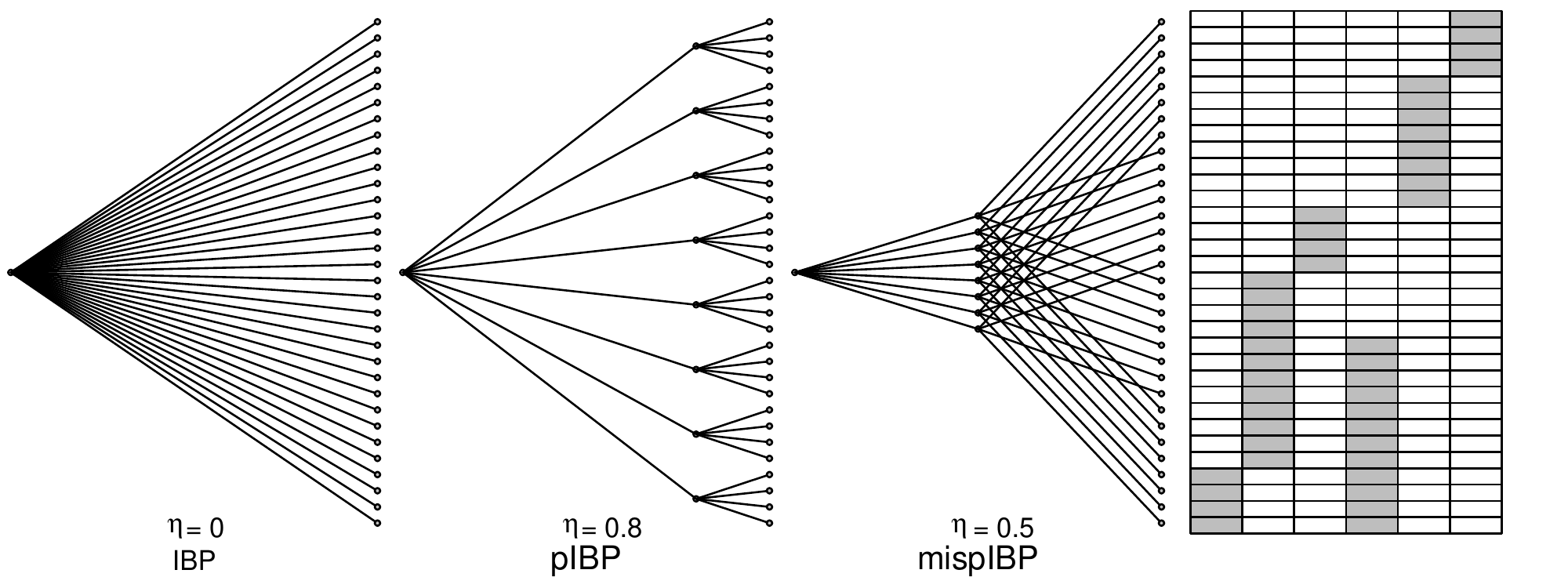}
\caption{The illustration of IBP, pIBP with an appropriate tree structure and pIBP with a mis-specified tree structure and the latent factor matrix $Z_0$ used in the first simulation. }
\label{sim1}
\end{figure}

\begin{table}[!ht]
\centering
\scriptsize
\def~{\hphantom{0}}
\caption{Simulation results: comparisons of IBP, pIBP with the appropriate tree prior and pIBP with the mis-specified tree prior (mispIBP).}
\begin{tabular}{cccccccc}
 \\
$(n, p)$  &  \multicolumn{2}{c}{IBP}    &   \multicolumn{2}{c}{pIBP}  &  \multicolumn{2}{c}{mispIBP}   \\[5pt] 
\hline 
          &  F-norm  & $\hat{K}$ &   F-norm  & $\hat{K}$  &  F-norm  & $\hat{K}$  \\[3pt] 
\hline       
(192,20) & 18.9 (15.2) & 8.1 (3.8) & 6.8 (2.3) & 6.7 (1.1) & 16.2 (9.1) & 6.6 (0.8) \\ [3pt] 
(288,20) & 20.3 (8.9) & 7 (1.9) & 10 (2.2) & 7 (0.9) & 19.3 (14.6) & 7 (0.9) \\       [3pt]  
(384,20) & 27.8 (7.4) & 7.5 (1.8) & 16 (7.7) & 7.9 (2.4) & 32.3 (5.8) & 7.8 (1.3)\\[3pt] 
\hline
(192,30) & 9.5 (6.9) & 6.6 (0.8) & 4.9 (3) & 6.1 (0.3) & 14.4 (15.3) & 6.8 (1.6) \\    [3pt] 
(288,30) & 14.2 (5.2) & 6.6 (0.5) & 7.9 (6.1) & 6.6 (1.4) & 13.2 (12.5) & 6.4 (0.6) \\ [3pt] 
(384,30) & 14.5 (8.2) & 6.7 (0.9) & 8 (4.8) & 6.4 (0.7) & 13.9 (9.7) & 6.7 (0.8) \\[3pt]   
\hline
(192,100) & 3.8 (2.3) & 5.9 (0.6) & 4 (2.2) & 5.8 (0.6) & 3.8 (2.2) & 5.9 (0.6) \\   [3pt]      
(288,100) & 5.5 (2.3) & 5.8 (0.5) & 5.2 (2) & 5.8 (0.6) & 5.3 (2.1) & 5.8 (0.5) \\    [3pt]     
(384,100) & 6 (3.4) & 6 (0.6) & 5.5 (3.9) & 6.2 (0.9) & 5.7 (3.4) & 6 (0.8) \\    [3pt]         
\hline
(192,200) & 3.8 (1.8) & 5.8 (0.6) & 3.8 (1.9) & 5.5 (1.1) & 3.8 (1.9) & 5.5 (1.1) \\     [3pt]  
(288,200) & 4.8 (2.3) & 5.7 (0.5) & 4.8 (2.3) & 5.7 (0.5) & 4.9 (2.4) & 5.7 (0.5) \\  [3pt]     
(384,200) & 5 (2.4) & 5.6 (0.6) & 4.7 (2.6) & 5.6 (0.5) & 4.6 (2.5) & 5.7 (0.6) \\[3pt]      
\hline
\end{tabular}
\begin{flushleft}
The performance is measured by estimation errors in terms of the normalized Frobenius norm of the feature similarity matrix ${n^{-1/2}{||ZZ^{\T}-Z_0Z_0^{\T}||_F}}$ (F-norm), and the number of estimated latent factors $\hat{K}$. Numbers in parentheses are the standard deviations across the 40 independent replicates. In the above models, $\sigma_{A,0}^2 = 1$, $\sigma_{X,0}^2 = $ 0$\cdot$5, $K_0=9$, results are based on 1000 Markov chain Monte Carlo steps.
\end{flushleft}
\label{tab1}
\end{table}

In the second simulation, we used the similarity data to construct pIBP prior. Nine models presented in Table \ref{tab2} are considered. For each model, we generated an $n \times {K_0}$ binary matrix $Z_0$ with 4 columns sampled from a $\text{Bernoulli}(\text{0$\cdot$3})$ and 5 columns with fixed structure. For IBP, no prior of the group structure is given. For pIBP, we first apply a hierarchical cluster analysis with complete linkage on the rows of $Z_0$ and then use its output dendrogram as the tree in the pIBP prior (see Figure \ref{sim2}). In our analysis, we constructed our prior based on the true knowledge of $Z_0$ in order to investigate whether the correct structural information will improve the performance through pIBP priors. In practice, such trees need to be constructed from external sources.  For mispIBP, the tree prior was constructed in the same way as pIBP but using a random permutation of $Z_0$ on rows in the clustering. In this setting, mispIBP represents totally incorrect information. 
Similar as  the previous simulation, we evaluated the performance by ${n^{-1/2}||ZZ^{\T}-Z_0Z_0^{\T}||_F}$ and the truncated number of estimated  latent features (Table \ref{tab2}). When $p$ is small, pIBP outperforms IBP in all cases.  When $p$ is adequately large ($p=60$ in this setting), the inference is less influenced by the prior information. 

\begin{figure}
\centering
\includegraphics[width=5in]{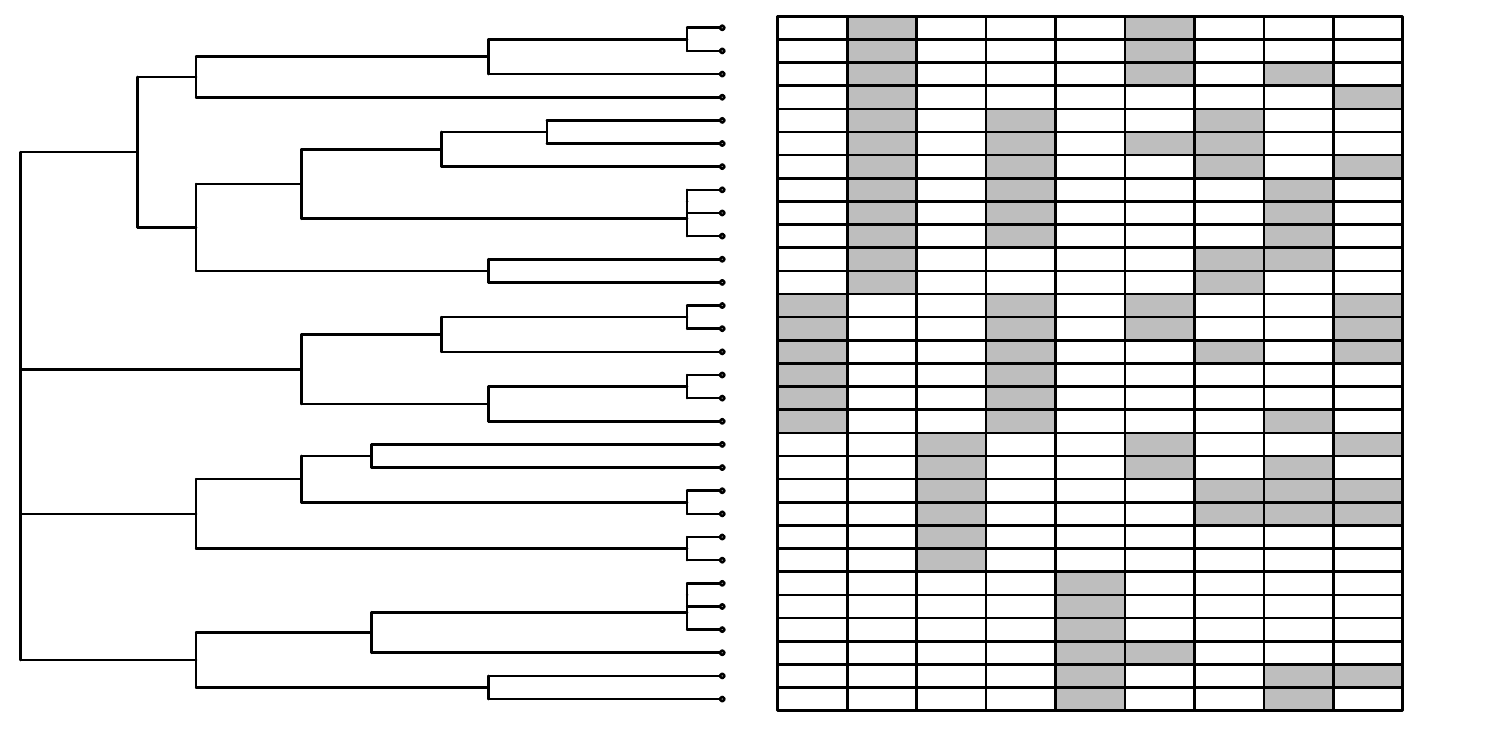}
\caption{The illustration of the latent factor matrix $Z_0$ and tree prior constructed from the hierarchical clustering analysis of $Z_0$ in the second simulation. }
\label{sim2}
\end{figure}

\begin{table}[!ht]
\centering
\scriptsize
\def~{\hphantom{0}}
\caption{Simulation results: comparisons of IBP and pIBP with the tree prior from the dendrogram of a hierarchical clustering on $Z_0$.}%
\begin{tabular}{cccccccc}
 \\
$(n, p)$  &  \multicolumn{2}{c}{IBP}    &   \multicolumn{2}{c}{pIBP}  &  \multicolumn{2}{c}{mispIBP}  \\[4pt] 
\hline 
          &  F-norm  & $\hat{K}$ &   F-norm  & $\hat{K}$  &  F-norm  & $\hat{K}$  \\ [3pt]
\hline 
(120, 15) & 28.5 (6) & 22.5 (1.6) & 11.4 (6.4) & 17 (3.9) & 31.1 (10.5) & 23.6 (3.4) \\[3pt] 
(180, 15) & 30.4 (3.9) & 21.5 (1.4) & 11.9 (4.7) & 15.5 (2.9) & 31.2 (7.1) & 23.1 (3.1) \\[3pt] 
(240, 15) & 35 (7.2) & 18.5 (4.9) & 13.4 (2.3) & 17.8 (2.5) & 32.6 (4.3) & 24.6 (2) \\[3pt] 
\hline
(120, 30) & 11.8 (7.7)  & 11.9 (3.6) & 7 (2.3) & 11.7 (2.5) &  8.1 (3.5) & 11.6 (1.5) \\[3pt] 
(180, 30) & 13.9 (6.9) & 12.3 (3) & 9.2 (2.9) & 13.3 (2.7) & 12.1 (3.3) & 12.4 (1.8)\\ [3pt]       
(240, 30) & 15.9 (10.4) & 12.2 (3.3) & 10.7 (3) & 13.2 (2.2) & 18.2 (8.4) & 11.1 (1.4) \\ [3pt]  
\hline
(120, 60) & 7.3 (2.8) & 11.2 (1.5) & 6.7 (2.3) & 10.6 (1.5) & 7.6 (2.5) & 10.6 (1.5) \\ [3pt]      
(180, 60) & 9.6 (2.5) & 11.7 (2.2) & 8.1 (2.5) & 11.1 (2.3) & 9.4 (3.9) & 10.8 (1.2) \\  [3pt]     
(240, 60) & 9.4 (3.2) & 11.5 (2.4) & 9.3 (2.2) & 10.8 (1.6)& 11.7 (4.2) & 11.3 (1.7) \\ [3pt] 
\hline
\end{tabular}
\begin{flushleft}
The performance is based on 40 independent replicates, each with 1000 Markov chain Monte Carlo steps.
\end{flushleft}
\label{tab2}
\end{table}

\section{Applications of pIBP in the Integrative Cancer Genomics Analysis}\label{sec:real}

Cancer research has been revolutionized by recent advances in high through-put technologies. Diverse types of genomics data, e.g., DNA, RNA, and epigenetic, have been profiled for different tumor types \citep{nik2012life,muzny2012comprehensive,bell2011integrated,mll2012comprehensive}. These data have revealed that substantial heterogeneities exist across tumor types, across individuals within the same tumor types and even within an individual tumor. However, the tumor heterogeneity at somatic level has not been explicitly explored in the integrative analysis.

Here we propose to use  binary factor model to integrate somatic mutation and gene expression data based on  pIBP prior. Our working hypothesis is that gene expression profiles of a cancer patient may be predicted by a set of latent factors that represent distinct molecular drivers. With this hypothesis, the more similar the somatic mutation profiles are between two cancer patients, the more similar their gene expression profiles are. Therefore, we build a pIBP prior based on somatic mutation data then specify it on the latent factors of gene expression data. Using this approach, we can investigate the gene expression data by taking into account the heterogeneities across cancer patients at somatic level. 

We consider studies on a specific cancer type/subtype, which collects somatic mutations from whole exome sequencing and gene expressions either from sequencing or microarrays for each sample. Somatic mutations can either be more narrowly defined as single nucleotide changes and small insertions/deletions, or more broadly defined to include changes at the copy number level. We denote the detected somatic mutations for a group of samples by a binary matrix $S=(s_{il})_{n\times m}$, with $s_{il}$ indicating the mutation status of the $l$th gene on the $i$th individual, as an external resource to construct the tree prior. When subclonality information is available, $s_{il}$ may be expressed as a continuous measure between 0 and 1, representing the percentage of the cells containing mutations at the $l$th gene. 

As for using a tree structure to express the relationships of individuals using the somatic mutation data, we propose to construct either logic tree or dendrogram tree. The logic tree prior is constructed as a logic tree based on the presence/absence of a set of somatic mutations. In this case, each node represents the status of a specific mutation. The dendrogram tree prior is adapted from the dendrogram tree of a hierarchical clustering on the somatic profiles $S=(s_{il})_{n\times m}$. In such a tree, the non-leaf nodes have no explicit meaning but represent a local cluster of individuals. When the order of mutation acquisitions and the effects of specific mutations are unknown, the dendrogram tree provides a measure of the overall similarities between individuals.

We analyzed the TCGA BRCA Level 3 dataset generated by \citep{mll2012comprehensive} (downloaded from cBio \citep{cerami2012cbio}) using the dendrogram tree construction strategy. We focused on 134 samples categorized as HER2 or Basal-like subtypes. Among these two subtypes, HER2 subtype is relatively well characterized and has effective clinical treatments. The basal-like subtype, which is also known as triple-negative breast cancers (TNBCs, lacking expression of ER, progesterone receptor (PR) and HER2), is poorly understood, with only chemotherapy as the main therapeutic option \citep{mll2012comprehensive}. Characterization of the basal-like subtype at the molecular level has important clinical implications. We built a tree prior from the dendrogram of a hierarchical clustering analysis with the frequent mutations in breast cancer including AKT1, CDH1, GATA3, MAP3K1, MLL3, PIK3CA, PIK3R1, PTEN, RUNX1 and TP53. For expression data, genes having top 300 MAD across samples were kept and centered. We ran 10 Markov chains. No substantial difference was observed across runs and we chose the one with largest posterior probability as the final result. Figure \ref{fig4} shows the input tree prior, subtype information and the inferred latent feature matrix $[Z]$.  

In our samples, the basal-like and HER2 samples display different and almost complementary patterns in their possession of the first two features. 74 of 81 Basal-like samples exhibit the first feature and 79 of 81 are deplete with the second feature. In contrast, 43 of 53 HER2 samples are deplete with the first feature and 31 of 53 exhibit the second feature. For the first feature, the top 10 genes with the largest loadings include MRPL9, PUF60, SCNM1, EIF2C2, BOP1, MTBP, DEDD, PHF20L1, HSF1 and HEATR1. Among these, BOP1 is involved in ribosome biogenesis and contributes to genomic stability, deregulation of which leads to altered chromosome segregation \citep{killian2006contribution}; MTBP inhibits cancer metastasis by interacting with MDM2 \citep{chene2003inhibiting}; DEDD interacts with PI3KC3 to activate autophagy and attenuate epithelialÐmesenchymal transition in cancer \citep{lv2012dedd}; and HSF1 has been proposed as a predictor of survival in breast cancer \citep{van2002gene}. EIF2C2, PUF60 and PHF20L1 have been reported as prognostic markers in ovarian cancer \citep{ramakrishna2010identification,wrzeszczynski2011identification}, which is consistent with the recent discovery that basal-like breast tumours with high-grade serous ovarian tumours share many molecular commonalities \citep{mll2012comprehensive}. These basal-like specific genes may potentially become novel therapeutic targets or prognostic markers. For the second feature, the top 10 genes with the largest loadings include STARD3, MED1, PSMD3, GRB7, ORMDL3, WIPF2, CASC3, RPL19, SNF8 and AMZ2. Among these, overexpressions of STARD3, PSMD3, GRB7, CASC3 and RPL19 have been reported in HER2-amplified breast cancer cell lines \citep{arriola2008genomic}; MED1 is required for estrogen receptor-mediated gene transcription and breast cancer cell growth \citep{zhang2011arginine}. As revealed by principal component analysis based on gene expression (Figure \ref{fig4}), these genes weighing high on first two latent features have discriminating power on Basal-like and HER2 samples.

\begin{figure}
\centering
\includegraphics[width=5in]{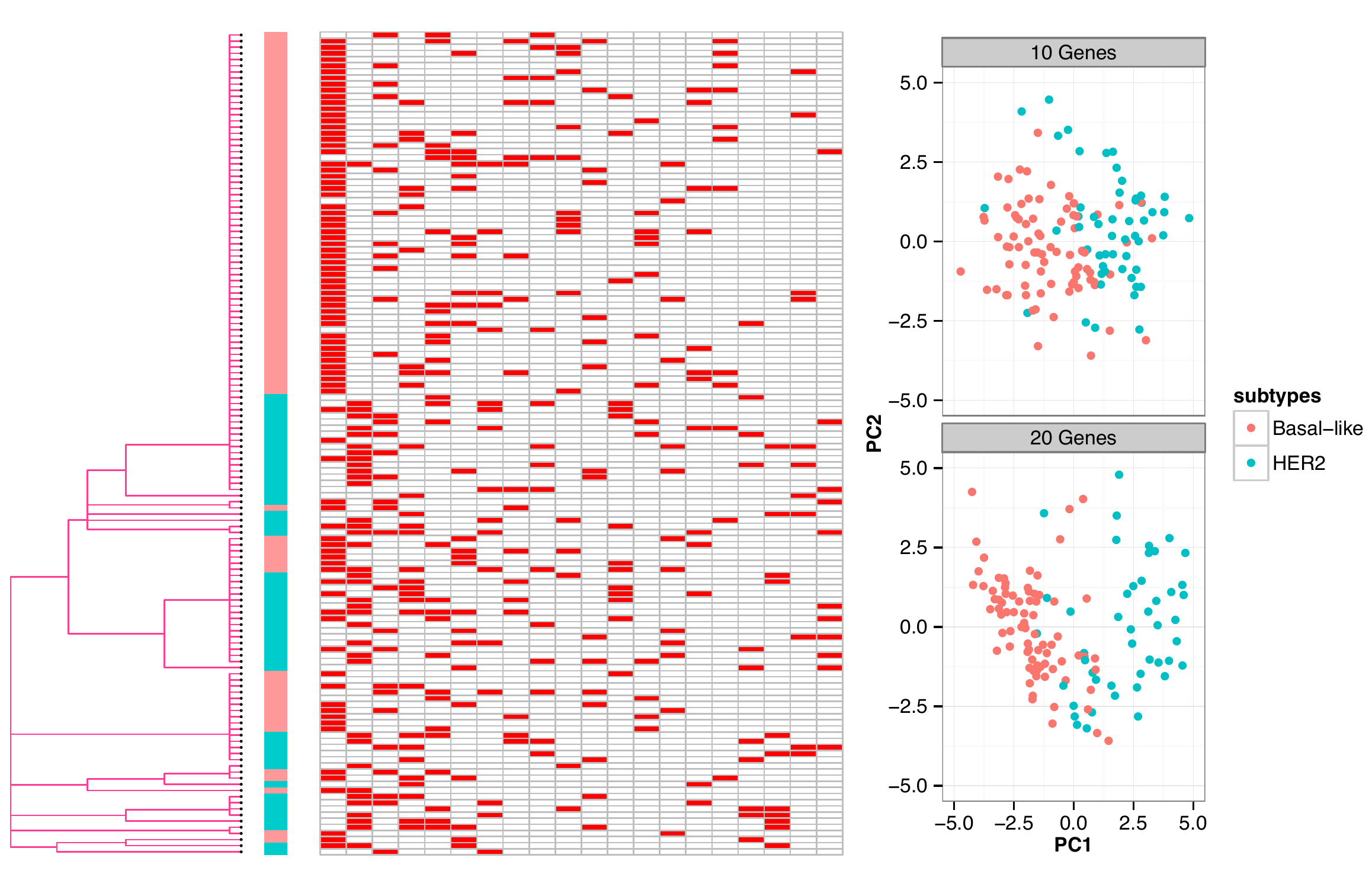}
\caption{A graph showing the dendrogram tree prior (left), the inferred latent factor matrix $[Z]$ (middle, only first 20 columns shown) and PCA analysis of Basal-like (Red) and HER2 (Green) based on genes with top loading on latent factors (topright, with a set of 10 genes from first factor; bottomright, with a set of 20 genes from first two factors) for TCGA BRCA dataset.}
\label{fig4}
\end{figure}

Furthermore, we found that the status of the fifth and sixth features was strongly associated with disease recurrence in our samples as revealed by survival analysis (Figure \ref{fig5} shows the Kaplan--Meier plot). Samples with the fifth feature have a higher probability of recurrence than those without it, with a p-value of 0$\cdot$0068, whereas samples without the sixth feature have a higher probability of recurrence than those with it, with a p-value of 0$\cdot$00084. Examinations of the loadings on these two features identified RMDN1, ARMC1, TMEM70, VCPIP1, TCEB1, MTDH, EBAG9, MRPL13, UBE2V2, FAM91A1 and RRS1 on the fifth feature and TRIM11, COMMD5, PYCRL, TIGD5, MRPL55, LSM1, SETDB1, CNOT7, PROSC, DEDD and HSF1 on the sixth feature. Among these, the prognosis significance of some has been discussed before, for example, MTDH activation by 8q22 genomic gain promotes chemoresistance and cetastasis of poor-prognosis breast cancer \citep{hu2009mtdh}; EBAG9 (RCAS1) is associated with ductal breast cancer progression \citep{rousseau2002rcas1}. The other  genes may serve as candidate tumor progression markers. 

In comparison, we analyzed the same 134 breast cancer samples with the expression profiles of 300 genes and the mutation status of 11 genes with IBP prior. The resulting latent factor matrix is less sparse than that of pIBP, which offers compromised interpretability (See Supplementary Figure 1). Moreover, the above reported features were not recovered by IBP prior, suggesting the integration of somatic mutations might lead to better understanding of gene expression.

\begin{figure}
\centering
\includegraphics[width=5in]{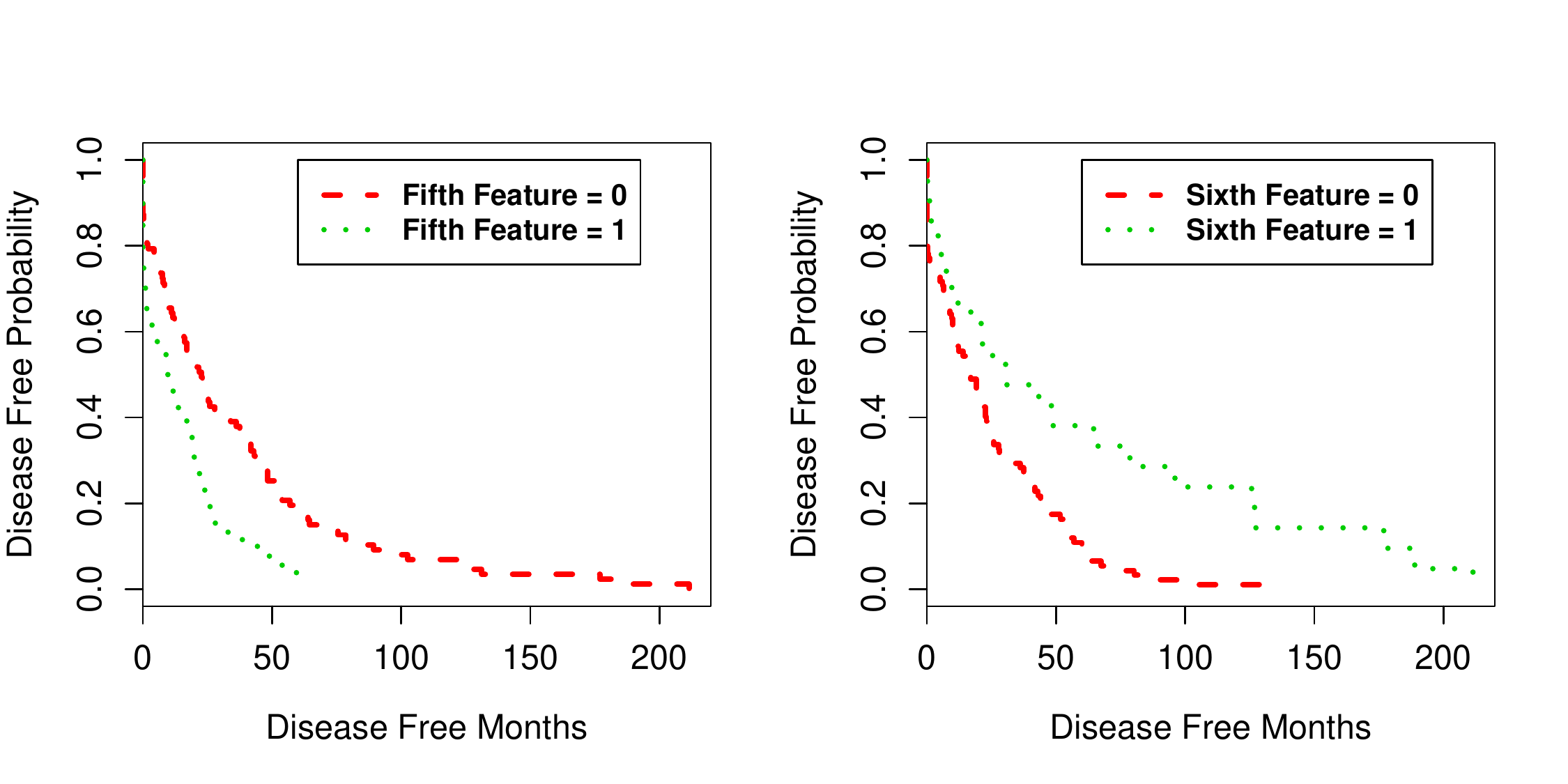}
\caption{A Kaplan--Meier plot for groups with different status of the fifth and sixth feature inferred from TCGA BRCA dataset.}
\label{fig5}
\end{figure}

\section{Discussion}\label{sec:disc}

\subsection{Related Work on Factor Models}

This paper attempts to provide a theoretical foundation for the widely used IBP and pIBP priors. We illustrate the performance of the priors through a simple binary factor model. To the best of our knowledge, there are only a few literatures on posterior rates of contraction for factor models and its alternative form principal component analysis (PCA). \cite{pati2014posterior} is the first work to consider posterior contraction rates for sparse factor models. \cite{gao2013rate} derives rate-optimal posterior contraction for sparse PCA. Both results achieve the frequentist minimax rates (up to a logarithmic factor for the first work).
Frequentist estimation in factor models include \cite{fan2008high}, \cite{fan2011high} and \cite{fan2013large}.

Minimax rates for factor models usually appear in the literature in the form of principal component analysis. For example, minimax rates for sparse PCA are derived by \cite{birnbaum2013minimax}, \cite{cai2013sparse}, \cite{cai2013optimal} and \cite{vu2013minimax} under various settings.

For binary factor models, minimax rates are not available in the literature, and it cannot be easily derived from the existing results. In the current binary factor model setting, there are two main points that deviate from the settings considered in the literature. First, the largest eigenvalue of the matrix $Z_0Z_0^{\T}+I$ may diverge as $n\rightarrow\infty$ in the extreme case, while most minimax rates in the literature for covariance estimation assume bounded spectrum. Second, the binary factor model only takes value in $\{0,1\}$, which distinguishes itself from ordinary factor models. The results in this paper suggest at least two open problems. First, what is the minimax rate of the binary factor model? Second, is IBP or pIBP rate-optimal? If not, what is the best rate of contraction that can be achieved by the posterior distribution?

\subsection{Approximate Group Structure}

Theorem \ref{thm:postpibp} states the posterior contraction rates of pIBP under the model of a two-group structure through the factor decomposition (\ref{eq:decomp}). Such characterization of group structure is exact in the sense that even only one person in $S_1$ possesses a factor that is mostly possessed by people in $S_2$, that factor is classified as a common factor, contributing to the total $K_0^*$. Therefore, in many real cases the exact two-group structure is violated and we can easily get $K_0^*=K_0$, thus losing the advantage of using pIBP.

In this section, we present a result to demonstrate that pIBP still gains advantage over IBP even when $K_0^*=K_0$ but the two-group structure approximately holds. We say $Z_0$ has an approximate two-group structure if there exists a binary matrix $Z^*$ of the same size such that the number $K_0^*$ associated with $Z^*$ is bounded by $O(K_0^{1-\beta})$ and $\left\|Z_0Z_0^{\T}-Z^*(Z^*)^{\T}\right\|_F$ is small. In other words, $Z_0$ may have a large $K^*_0$, but it is close to a binary factor matrix whose $K_0^*$ is small. The following theorem is an oracle inequality for pIBP under the posterior distribution.

\begin{thm}\label{thm:oracle}
Let $Z_0\in\{0,1\}^{n\times K_0}$ be an arbitrary binary factor matrix, and let $Z^*\in\{0,1\}^{n\times K_0}$ be a binary factor matrix with a well specified group structure such that its $K_0^*\lesssim K_0^{1-\beta}$ for $\beta\in (0,1)$.
Under the assumption of Theorem \ref{thm:postpibp},
$$
 E_{Z_0}\left[\Pi\Bigg(\left\|ZZ^{\T}-Z_0(Z_0)^{\T}\right\|_F^2\leq   M\left(\frac{n^3K_0^{4-2\beta}}{p}+n^2K_0^2||Z_0Z_0^T-Z^*(Z^*)^{\T}||^4\right)\Big|X\Bigg)\right]
$$
$$\geq  1-\exp\left(-C'nK_0^{2(1-\beta)}\right)-\frac{2}{p},$$
for some constants $M,C'>0$.
\end{thm}

In the case when $Z_0$ has an exact two-group structure, we may choose $Z^*=Z_0$ so that
 $\left\|Z_0Z_0^{\T}-Z^*(Z^*)^{\T}\right\|_F=0$. Then it reduces to the result in Theorem \ref{thm:postpibp}. Otherwise, we may choose a $Z^*$ with an exact two-group structure to approximate $Z_0$. In this case, the posterior distribution contracts to the truth with a rate consisting of two parts. The first part can be viewed as the estimation error of a binary factor matrix $Z^*$ with an exact two-group structure. The second part is the approximation error for the true binary factor matrix $Z_0$  by $Z^*$. Note that the rate of convergence for IBP in Theorem \ref{thm:postibp} is $\frac{K_0^4n^3}{p}$. Therefore, as long as
 $$n^2K_0^2\left\|Z_0Z_0^{\T}-Z^*(Z^*)^{\T}\right\|_F^4=o\left(\frac{K_0^4n^3}{p}\right),$$
 pIBP still converges faster than IBP if the true binary factor matrix $Z_0$ has an approximate two-group structure.

Let us consider the following example to illustrate Theorem \ref{thm:oracle}. Let $Z_0\in\{0,1\}^{n\times K_0}$ be a binary factor matrix which generates the data. Among the $K_0-K_0^*=K_{01}+K_{02}$ factors that possess approximate group structures, there are $K_{01}$ factors belonging to $S_1$ and $K_{02}$ factors belonging to $S_2$. In addition, for some small $\delta\in (0,1)$, $n^{\delta}$ people in $S_1$ can possess a constant number of factors belonging to $S_2$, and $n^{\delta}$ people in $S_2$ can possess a constant number of factors belonging to $S_1$.
We call this situation a $\delta$-approximate two-group structure. By zeroing out these entries, we obtain a binary factor matrix $Z^*\in\{0,1\}^{n\times K_0}$ with an exact two-group structure, whose factor decomposition is $K_0=K_{01}+K_{02}+K_0^*$. In other words, for $Z^*$, there are $K_{01}$ factors exclusively belonging to $S_1$ and $K_{02}$ factors exclusively belonging to $S_2$. The approximation error is bounded by $\left\|Z_0Z_0^{\T}-Z^*(Z^*)^{\T}\right\|^2_F\lesssim ||Z_0||^2||Z_0-Z^*||_F^2\lesssim n^{\delta}||Z_0||^2$, where $||\cdot||$ denotes the spectral norm of a matrix, which is its largest singular value. We summarize this example in the following corollary.

\begin{corollary}
Under the setting of Theorem \ref{thm:oracle}, let $Z^*$ have a factor decomposition satisfying $K_0^*\lesssim K_0^{1-\beta}$, then as long as $n^{2\delta}=o\left(\frac{nK_0^{2}}{p||Z_0||^4}\right)$, we have
$$
 E_{Z_0}\left[\Pi\Bigg(\left\|ZZ^{\T}-Z_0(Z_0)^{\T}\right\|_F^2\leq   \epsilon_{n,p}^2\Big|X\Bigg)\right]
\geq  1-\exp\left(-C'nK_0^{2(1-\beta)}\right)-\frac{2}{p},
$$
for some positive sequence $\epsilon_{n,p}^2=o\left(\frac{K_0^4n^3}{p}\right)$ and some constant $C>0$.
\end{corollary}
The corollary provides an example that pIBP converges at a faster rate than that of IBP when $Z_0$ satisfies the $\delta$-approximate two-group structure. The quantity $||Z_0||$ quantifies the sparsity of the binary factor matrix $Z_0$. In many applied situations, the true binary factor matrix $Z_0$ has a sparse structure \citep{griffiths2011indian,knowles2011nonparametric,carvalho2008high}. This leads to a small $||Z_0||$.


\bibliographystyle{plainnat}
\bibliography{FeatureTree}

\newpage
\thispagestyle{empty}

\setcounter{page}{1}
\begin{center}
\MakeUppercase{\large Supplement to ``Posterior Contraction Rates of the Phylogenetic Indian Buffet Processes"}
\end{center}

This manuscript serves as the supplementary material to the paper \cite{chen14}. Section \ref{sec:proof} and Section \ref{sec:tech} present technical proofs of the main results of the paper. Section \ref{sec:apply} presents an alternative analysis of the real data studied in \cite{chen14}. Given a matrix $A=(a_{ij})_{m\times n}$, its sup-norm is defined as $||A||_{\infty}=\max_{ij}|a_{ij}|$ and its spectral norm is defined $||A||=s_{\max}(A)$, where $s_{\max}(\cdot)$ is the largest singular value of a matrix. The notation $\mathbb{P}$ and $\mathbb{E}$ stand for generic probability and expectation operators when the associated distribution is clear from the context. We use $C$ and its variants such as $C'$ and $C_1$ to denote generic constants, which may vary from line to line.

\appendix

\section{Proofs} \label{sec:proof}

\subsection{Preparatory lemmas}

\begin{lemma} \label{lem:testing}
There is some constant $C>0$ such that for any $t>0$,
$$P_{Z}\left\{\left\|\frac{1}{p}XX^T-(ZZ^T+I)\right\|_F>t\right\}\leq\exp\left\{-Cp\min\Bigg(\frac{t^2}{n^2||ZZ^T+I)||_{\infty}^2},\frac{t}{n||ZZ^T+I)||_{\infty}}\Bigg)+2\log n\right\},$$
and
$$P_{Z}\left\{\left\|\frac{1}{p}XX^T-(ZZ^T+I)\right\|_{\infty}>t\right\}\leq\exp\left\{-Cp\min\Bigg(\frac{t^2}{||ZZ^T+I)||_{\infty}^2},\frac{t}{||ZZ^T+I)||_{\infty}}\Bigg)+2\log n\right\}.$$
\end{lemma}

\subsection{Proofs of Theorem 4.1 and Lemma 4.1}

For notational simplicity, we write $\epsilon$ for $\epsilon_{n,p}$, with the dependency on $n$ and $p$ being implicit. 

\begin{proof}[Proof of Theorem 4.1]
The posterior distribution, according to Bayes formula, is
$$\Pi(U|X)=\frac{\int_U\frac{p(X|Z)}{p(X|Z_0)}d\Pi([Z])}{\int\frac{p(X|Z)}{p(X|Z_0)}d\Pi([Z])}.$$
The denominator has lower bound
$$\int\frac{p(X|Z)}{p(X|Z_0)}d\Pi([Z])\geq \int_{\{||ZZ^T-Z_0Z_0^T||_F^2=0\}}\frac{p(X|Z)}{p(X|Z_0)}d\Pi([Z])=\Pi\Big(||ZZ^T-Z_0Z_0^T||_F^2=0\Big).$$
The above equality is because $p(X|Z)=p(X|Z_0)$ when $||ZZ^T-Z_0Z_0^T||_F=0$.
Thus, we have
\begin{eqnarray*}
E_{Z_0}\Pi(U|X) &\leq& E_{Z_0}\phi + E_{Z_0}\Pi(U|X)(1-\phi) \\
&\leq& E_{Z_0}\phi + \frac{E_{Z_0}\Bigg(\int_U\frac{p(X|Z)}{p(X|Z_0)}d\Pi([Z])(1-\phi)\Bigg)}{\Pi\Big(||ZZ^T-Z_0Z_0^T||_F^2=0\Big)} \\
&=& E_{Z_0}\phi + \frac{\int_U E_Z(1-\phi)d\Pi([Z])}{\Pi\Big(||ZZ^T-Z_0Z_0^T||_F^2=0\Big)} \\
&\leq& E_{Z_0}\phi+\frac{1}{\Pi\Big(||ZZ^T-Z_0Z_0^T||_F^2=0\Big)}\sup_{Z\in U}E_Z(1-\phi),
\end{eqnarray*}
where the equality above is due to Fubini's Theorem. Therefore, the proof is complete.
\end{proof}

\begin{proof}[Proof of Lemma 4.1]
We consider the following test.
$$H_0: Z=Z_0,\quad H_1: ||ZZ^T-Z_0Z_0^T||_F> \sqrt{M}\epsilon.$$
The alternative region has decomposition
\begin{eqnarray*}
H_1 &\subset& \left\{||ZZ^T-Z_0Z_0^T||_F> \sqrt{M}\epsilon, ||ZZ^T+I||_{\infty}\leq 4(K_0+1)\right\} \\
&& \cup \bigcup_{l\geq 1}\left\{4l(K_0+1)<||ZZ^T+I||_{\infty}\leq 4(l+1)(K_0+1)\right\} \\
&=& \bigcup_{l=0}^{\infty}H_{1l}
\end{eqnarray*}
Define the testing functions
$$\phi_0=\mathbb{I}\left\{\left\|\frac{1}{p}XX^T-(Z_0Z_0^T+I)\right\|_F>\frac{1}{2}\sqrt{M}\epsilon\right\},$$
$$\phi_l=\mathbb{I}\left\{\left\|\frac{1}{p}XX^T\right\|_{\infty}>2l(K_0+1)\right\},\quad \text{for each }l.$$
Then, by Lemma \ref{lem:testing} and the fact that $||Z_0Z_0^T+I||_{\infty}\leq K_0+1\leq 2K_0$, we have
$$E_{Z_0}\phi_0\leq\exp\left\{-Cp\min\Bigg(\frac{M\epsilon^2}{n^2K_0^2},\frac{\sqrt{M}\epsilon}{nK_0}\Bigg)+2\log n\right\},$$
and
\begin{eqnarray*}
E_{Z_0}\phi_l &=& P_{Z_0}\left\{\left\|\frac{1}{p}XX^T\right\|_{\infty}>2l(K_0+1)\right\}\\
&\leq& P_{Z_0}\left\{\left\|\frac{1}{p}XX^T-(Z_0Z_0^T+I)\right\|_{\infty}>2l(K_0+1)-\left\|Z_0Z_0^T+I\right\|_{\infty}\right\} \\
&\leq& P_{Z_0}\left\{\left\|\frac{1}{p}XX^T-(Z_0Z_0^T+I)\right\|_{\infty}>l(K_0+1)\right\} \\
&\leq& \exp\Big(-Clp+2\log n\Big),
\end{eqnarray*}
where the second inequality above is by $\left\|Z_0Z_0^T+I\right\|_{\infty}\leq K_0+1\leq l(K_0+1)$, and the last inequality above is 
by Lemma \ref{lem:testing}. We also have for any $Z\in H_{10}$,
\begin{eqnarray*}
E_Z(1-\phi_0) &=& P_Z\left\{\left\|\frac{1}{p}XX^T-(Z_0Z_0^T+I)\right\|_F\leq\frac{1}{2}\sqrt{M}\epsilon\right\} \\
&\leq& P_Z\left\{||ZZ^T-Z_0Z_0^T||_F-\left\|\frac{1}{p}XX^T-(ZZ^T+I)\right\|_F\leq \frac{1}{2}\sqrt{M}\epsilon\right\} \\
&\leq& P_Z\left\{\left\|\frac{1}{p}XX^T-(ZZ^T+I)\right\|_F>\frac{1}{2}\sqrt{M}\epsilon\right\} \\
&\leq& \exp\left\{-Cp\min\Bigg(\frac{M\epsilon^2}{n^2K_0^2},\frac{\sqrt{M}\epsilon}{nK_0}\Bigg)+2\log n\right\},
\end{eqnarray*}
where the last inequality is by Lemma \ref{lem:testing} and the fact that $||ZZ^T+I||_{\infty}\leq 4(K_0+1)\leq 8K_0$ for any $Z\in H_{10}$.
Taking supreme over $Z\in H_{10}$, we get
$$\sup_{Z\in H_{10}}E_Z(1-\phi_0)\leq \exp\left\{-Cp\min\Bigg(\frac{M\epsilon^2}{n^2K_0^2},\frac{\sqrt{M}\epsilon}{nK_0}\Bigg)+2\log n\right\}.$$
For any $Z\in H_{1l}$, we have
\begin{eqnarray*}
E_Z(1-\phi_l) &=& P_Z\left\{\left\|\frac{1}{p}XX^T\right\|_{\infty}\leq 2l(K_0+1)\right\} \\
&\leq& P_Z\left\{||ZZ^T+I||_{\infty}-\left\|\frac{1}{p}XX^T-(ZZ^T+I)\right\|_{\infty}\leq 2l(K_0+1)\right\} \\
&\leq& P_Z\left\{\left\|\frac{1}{p}XX^T-(ZZ^T+I)\right\|_{\infty}>2l(K_0+1)\right\} \\
&\leq& \exp\Big(-Cp+2\log n\Big),
\end{eqnarray*}
where the last inequality above uses Lemma \ref{lem:testing} and the fact that $||ZZ^T+I||_{\infty}\leq 4(l+1)(K_0+1)$ for $Z\in H_{1l}$, and the second last inequality uses the fact that $||ZZ^T+I||_{\infty}>4l(K_0+1)$ for all $Z\in H_{1l}$. Taking supreme over $Z\in H_{1l}$, we obtain
$$\sup_{Z\in H_{1l}}E_Z(1-\phi_l)\leq  \exp\Big(-Cp+2\log n\Big).$$
Define $\phi=\max_{l}\phi_l$, we have
\begin{eqnarray*}
&& E_{Z_0}\phi+\sup_{Z\in H_1}E_Z(1-\phi) \\
&=& E_{Z_0}\max_l\phi_l+\max_l\sup_{Z\in H_{1l}}E_Z(1-\phi) \\
&\leq& \sum_lE_{Z_0}\phi_l + \max_l\sup_{Z\in H_{1l}}E_Z(1-\phi_l) \\
&\leq& 2\exp\left\{-Cp\min\Bigg(\frac{M\epsilon^2}{n^2K_0^2},\frac{\sqrt{M}\epsilon}{nK_0}\Bigg)+2\log n\right\}\\
&&+\sum_{l=1}^{\infty}\exp\Big(-Clp+2\log n\Big)+\exp\left(-Cp+2\log n\right) \\
&\leq& 2\exp\left\{-Cp\min\Bigg(\frac{M\epsilon^2}{n^2K_0^2},\frac{\sqrt{M}\epsilon}{nK_0}\Bigg)+2\log n\right\}+\exp\Big(-C'p+2\log n\Big).
\end{eqnarray*}
Thus, the proof is complete.
\end{proof}

\subsection{Proof of Theorems 4.2-4.4}

\begin{proof}[Proof of Theorem 4.2]

Without loss of generality, we assume $n$ is even in the proof.
First, note that the event $\left\{||ZZ^T-Z_0Z_0^T||_F^2=0\right\}$ is implied by $\{||Z-Z_0||_F^2=0\}$ for any column ordering of $Z_0$. Therefore, we have
$$\Pi\Big(||ZZ^T-Z_0Z_0^T||_F^2=0\Big)\geq P\Big(||Z-Z_0||_F^2=0\Big),$$
with $P$ being any probability measure on $Z$ whose image measure under the map $Z\mapsto [Z]$ is pIBP. We choose $P$ to be the stick-breaking representation described in Section 3.2. That is, under probability $P$, we first sample $\{p_k\}$ according to (2) in \cite{chen14}, and then given $\{p_k\}$, $Z$ is sampled according to the two-group tree structure for each column. Define $r_{1k}$ and $r_{2k}$ to be the group nodes for the first and the second group, respectively, for each $k$. Then according to the stick-breaking representation of pIBP, $\{r_{1k}\}$ and $\{r_{2k}\}$ given $\{p_k\}$ are i.i.d. Bernoulli random variables with parameter $1-\exp(-\eta\gamma_k)$, where $\gamma_k=-\log(1-p_k)$. Then, $z_{ik}$ are sampled conditioning on $(r_{1k},r_{2k})$. When $r_{1k}=1$, we have $z_{ik}=1$ for all $i\in S_1$. When $r_{1k}=0$, $z_{ik}$ follows the Bernoulli distribution with parameter $1-\exp\big(-(1-\eta)\gamma_k\big)$ for all $i\in S_1$. The value of $r_{2k}$ determines the distribution of $z_{ik}$ for $i\in S_2$ in the same way.

We first study $P\Big(||Z-Z_0||_F^2=0\Big|\{v_k\},\alpha\Big)$ for  given $\{v_k\}$ and $\alpha$. We choose a particular ordering of columns of $Z_0$. Given the factor decomposition (5) in \cite{chen14}, let the first $K_0^*$ columns correspond to the group-shared factors, and the next $K_{01}+K_{02}$ columns correspond to the group specific factors. Then define the number of $1$'s in the $k$-th column of $Z_0$ by
$$m_k=\sum_{\{i:z_{0,ik}=1\}}z_{0,ik},\quad\text{for }k=1,...,K_0^*.$$
Define $M^*=\sum_{k=1}^{K_0^*}m_k$ to be the number of $1$'s in the first $K_0^*$ columns of $Z_0$. The quantity $||Z-Z_0||_F^2$ has four parts.
$$||Z-Z_0||_F^2=\sum_{k=1}^{K_0^*}U_k+\sum_{k=1}^{K_0^*}V_k+\sum_{k=K_0^*+1}^{K_0}\sum_{i=1}^n(z_{ik}-z_{0,ik})^2+\sum_{k=K_0+1}^{\infty}\sum_{i=1}^n z_{ik}.$$
where
$$U_k=\sum_{\{i:z_{0,ik}=0\}}z_{ik},\quad V_k=\sum_{\{i:z_{0,ik}=1\}}|z_{ik}-1|.$$
We observe that given $\{v_k\}$, the four terms are independent. Therefore
\begin{eqnarray}
\nonumber&& P\Big(||Z-Z_0||_F^2=0|\{v_k\},\alpha\Big) \\
\nonumber&=& P\Bigg(\sum_{k=1}^{K_0^*}U_k=0\Big|\{v_k\},\alpha\Bigg)\times P\Bigg(\sum_{k=1}^{K_0^*}V_k=0\Big|\{v_k\},\alpha\Bigg)\\
\label{eq:product}&&  \times P\Bigg(\sum_{k=K_0^*+1}^{K_0}\sum_{i=1}^n(z_{ik}-z_{0,ik})^2=0\Big|\{v_k\},\alpha\Bigg) \times  P\Bigg(\sum_{k=K_0+1}^{\infty}\sum_{i=1}^n z_{ik}=0\Big|\{v_k\},\alpha\Bigg).
\end{eqnarray}
We study the four terms separately. Define $\mathcal{H}=\left\{\frac{1}{4}\leq v_i\leq\frac{3}{4}, \text{ for }k=1,...,K_0\right\}$.
Then, for every $\{v_k\}\in\mathcal{H}$, we have
\begin{eqnarray}
\nonumber && P\Bigg(\sum_{k=1}^{K_0^*}U_k=0\Big|\{v_k\},\alpha\Bigg)\times P\Bigg(\sum_{k=1}^{K_0^*}V_k=0\Big|\{v_k\},\alpha\Bigg) \\
\nonumber &\geq& \Big(\exp\big(-\gamma_1(1-\eta)\big)\Big)^{nK_0^*-M^*}\Big(1-\exp\big(-\gamma_{K_0^*}(1-\eta)\big)\Big)^{M^*}\\
\nonumber && \times P\Big(r_{11}=...=r_{1K_0^*}=r_{21}=...=r_{2K_0^*}=0\Big|\{v_k\},\alpha\Big) \\
\nonumber &\geq&  \Big(\exp\big(-\gamma_1(1-\eta)\big)\Big)^{nK_0^*-M^*}\Big(1-\exp\big(-\gamma_{K_0^*}(1-\eta)\big)\Big)^{M^*}\times \exp\Big(-2K_0^*\gamma_1\eta\Big) \\
\nonumber &=& (1-p_1)^{(nK_0^*-M^*)(1-\eta)+2K_0^*\eta}\Bigg(1-(1-p_{K_0^*})^{1-\eta}\Bigg)^{M^*} \\
\label{eq:rev01} &\geq& (1-p_1)^{(nK_0^*-M^*)(1-\eta)+2K_0^*\eta}p_{K_0^*}^{M^*}(1-\eta)^{M^*} \\
\nonumber &\geq& 4^{-(nK_0^*-M^*)(1-\eta)}4^{-2K_0^*\eta}4^{-K_0^*M^*}(1-\eta)^{M^*} \\
\label{eq:rev02} &\geq& \exp(-Cn{K_0^*}^2)(1-\eta)^{nK_0^*},
\end{eqnarray}
where we have used the inequality $1-q^{\beta}\geq\beta(1-q)$ for $\beta,q\in (0,1)$ to derive (\ref{eq:rev01}). The inequality (\ref{eq:rev02}) is due to the bound $M^*\leq nK_0^*$. The third term of (\ref{eq:product}) is
\begin{eqnarray*}
&& P\Bigg(\sum_{k=K_0^*+1}^{K_0}\sum_{i=1}^n(z_{ik}-z_{0,ik})^2=0\Big|\{v_k\},\alpha\Bigg) \\
&\geq& \exp\Big(-n(K_{01}+K_{02})\gamma_{K_0^*}(1-\eta)/2\Big)\times P\Big(r_{1k}=1,r_{2k}=0,\text{ for }k=K_0^*+1,...,K_0^*+K_{01}\Big|\{v_k\},\alpha\Big) \\
&& \times P\Big(r_{1k}=0,r_{2k}=1,\text{ for }k=K_0^*+K_{01}+1,...,K_0^*+K_{01}+K_{02}\Big|\{v_k\},\alpha\Big) \\
&\geq& \exp\Big(-n(K_{01}+K_{02})\gamma_{K_0^*}(1-\eta)/2\Big)\times \Big(1-\exp(-\eta\gamma_{K_0})\Big)^{K_{01}+K_{02}} \times \exp\Big(-\eta(K_{01}+K_{02})\gamma_{K_0^*}\Big) \\
&\geq& (1-p_{K_0^*})^{(\eta+n(1-\eta)/2)(K_{01}+K_{02})}p_{K_0}^{K_{01}+K_{02}}\eta^{K_{01}+K_{02}} \\
&\geq& \left(1-(4/3)^{-K_0^*}\right)^{(\eta+n(1-\eta)/2)(K_{01}+K_{02})}4^{-K_0(K_0-K_0^*)}\eta^{K_{01}+K_{02}} \\
&=& \exp\left((\eta+n(1-\eta)/2)(K_{01}+K_{02})\log\left(1-(4/3)^{-K_0^*}\right)\right)4^{-K_0(K_0-K_0^*)}\eta^{K_{01}+K_{02}}\\
&\geq& \exp\left(-Cn\frac{K_0-K_0^*}{(4/3)^{K_0^*}}-CK_0(K_0-K_0^*)\right)\eta^{K_{01}+K_{02}},
\end{eqnarray*} 
where the last inequality is due to the fact that $\log(1-x)\geq-\delta x,\quad\text{for }|x|\leq 3/4$,
with $\delta>0$ being a universal constant.
The last term in the product (\ref{eq:product}) is
\begin{eqnarray}
\nonumber && P\Bigg(\sum_{k=K_0+1}^{\infty}\sum_{i=1}^n z_{ik}=0\Big|\{v_k\},\alpha\Bigg) \\
\nonumber &\geq& \prod_{k=K_0+1}^{\infty}\exp\Big(-n\gamma_k(1-\eta)\Big) \times P\Big(r_{1k}=r_{2k}=0,\text{ for }k>K_0\Big|\{v_k\},\alpha\Big) \\
\nonumber &\geq& \prod_{k=K_0+1}^{\infty}\exp\Big(-n\gamma_k(1-\eta)\Big) \times \prod_{k=K_0+1}^{\infty}\exp(-2\eta\gamma_k) \\
\nonumber &=& \prod_{k=K_0+1}^{\infty} (1-p_k)^{n(1-\eta)+2\eta} \\
\nonumber &\geq& \prod_{k=K_0+1}^{\infty}\left(1-(4/3)^{-k}\right)^{n(1-\eta)+2\eta} \\
\nonumber &=& \exp\left((n(1-\eta)+2\eta)\sum_{k=K_0+1}^{\infty}\log\left(1-(4/3)^{-k}\right)\right) \\
\label{eq:rev007}&\geq& \exp\left(-\delta(n(1-\eta)+2\eta)\sum_{k=K_0+1}^{\infty}(4/3)^{-k}\right) \\
\nonumber &=& \exp\left(-3\delta(4/3)^{-K_0}(n(1-\eta)+2\eta)\right) \\
\nonumber &\geq& \exp\left(-Cn\right)
\end{eqnarray}
where the inequality (\ref{eq:rev007}) is due to the fact that $\log(1-x)\geq-\delta x,\quad\text{for }|x|\leq 3/4$,
with $\delta>0$ being a universal constant. For a constant $\eta\in (0,1)$, we have $(1-\eta)^{nK_0^*}\geq\exp(-CnK_0^*)$ and $\eta^{K_{01}+K_{02}}\geq\exp(-C(K_0-K_0^*))$ and thus
\begin{eqnarray*}
&& {P}\Big(||Z-Z_0||_F^2=0|\{v_k\},\alpha\Big) \\
&\geq& \exp\left(-Cn({K_0^*}^2+1)-Cn\frac{K_0-K_0^*}{(4/3)^{K_0^*}}-CK_0(K_0-K_0^*)\right),
\end{eqnarray*}
for every $\{v_k\}\in\mathcal{H}$. Observe that the above argument also works by replacing $K_0$ and $K_0^*$ by $K_0+\kappa$ and $K_0^*+\kappa$ for any $\kappa\geq 0$. Thus, for a constant $\eta\in (0,1)$, we have
\begin{eqnarray}
\nonumber&& {P}\Big(||Z-Z_0||_F^2=0|\{v_k\},\alpha\Big) \\
\label{eq:rev-imp}&\geq& \exp\left(-Cn(({K_0^*+\kappa})^2+1)-Cn\frac{K_0-K_0^*}{(4/3)^{K_0^*+\kappa}}-C(K_0+\kappa)(K_0-K_0^*)\right),
\end{eqnarray}
for every $\{v_k\}\in\mathcal{H}$. When $\eta=0$, pIBP becomes IBP. Thus, the decomposition (\ref{eq:product}) becomes
\begin{eqnarray}
\nonumber&& P\Big(||Z-Z_0||_F^2=0|\{v_k\},\alpha\Big) \\
\nonumber&=& P\Bigg(\sum_{k=1}^{K_0}U_k=0\Big|\{v_k\},\alpha\Bigg)\times P\Bigg(\sum_{k=1}^{K_0}V_k=0\Big|\{v_k\},\alpha\Bigg)\times P\Bigg(\sum_{k=K_0+1}^{\infty}\sum_{i=1}^n z_{ik}=0\Big|\{v_k\},\alpha\Bigg).
\end{eqnarray}
Replacing $K_0^*$ by $K_0$ in (\ref{eq:rev02}), we have
$${P}\Big(||Z-Z_0||_F^2=0|\{v_k\},\alpha\Big)\geq \exp\left(-CnK_0^2\right),$$
for $\eta=0$ and every $\{v_k\}\in\mathcal{H}$. Note that this is a special case of (\ref{eq:rev-imp}) with $K_0^*=K_0$ and $\kappa=0$.
Finally, we have
\begin{eqnarray*}
&& P\Big(||Z-Z_0||_F^2=0\Big) \\
&\geq& P\Big(||Z-Z_0||_F^2=0\Big|\{v_k\}\in\mathcal{H},\alpha\in(1/2,2)\Big)\mathbb{P}\Big(\mathcal{H}\Big|\alpha\in (1/2,2)\Big)\mathbb{P}\Big(\alpha\in (1/2,2)\Big) \\
&\geq& P\Big(||Z-Z_0||_F^2=0\Big|\{v_k\}\in\mathcal{H},\alpha\in(1/2,2)\Big) \Bigg(\sup_{\alpha\in(1/2,2)}\frac{\alpha\text{B}(\alpha,1)}{(3/4)^{\alpha}-(1/4)^{\alpha}}\Bigg)^{-K_0-\kappa}\mathbb{P}\Big(\alpha\in (1/2,2)\Big)\\
&\geq& \exp(-C(K_0+\kappa))P\Big(||Z-Z_0||_F^2=0\Big|\{v_k\}\in\mathcal{H},\alpha\in(1/2,2)\Big) \\
&\geq&\exp\left(-Cn(({K_0^*+\kappa})^2+1)-Cn\frac{K_0-K_0^*}{(4/3)^{K_0^*+\kappa}}-C(K_0+\kappa)(K_0-K_0^*+1)\right),
\end{eqnarray*}
by plugging (\ref{eq:rev-imp}). Thus, the proof is complete.
\end{proof}

\begin{proof}[Proof of Theorem 4.3-4.4]
This is directly by combining Theorem 4.1, Lemma 4.1, Theorem 4.2 and the discussion after Theorem 4.2. For Theorem 4.3, we have
\begin{eqnarray*}
&& E_{Z_0}\Pi\left(\left\|ZZ^T-Z_0Z_0^T\right\|^2>M\epsilon^2 |X\right) \\
&\leq& \frac{\exp\left\{-Cp\min\Bigg(\frac{M\epsilon^2}{n^2K_0^2},\frac{\sqrt{M}\epsilon}{nK_0}\Bigg)+2\log n\right\}+\exp\Big(-Cp+2\log n\Big)}{\exp\left(-C_1nK_0^2\right)}.
\end{eqnarray*}
Taking $\epsilon^2=\frac{K_0^4n^3}{p}$, we have $p\min\Bigg(\frac{M\epsilon^2}{n^2K_0^2},\frac{\sqrt{M}\epsilon}{nK_0}\Bigg)\asymp nK_0^2$ under the assumption of Theorem 4.3 that $nK_0^2=o(p)$. Thus, for some sufficiently large $M$,
$$E_{Z_0}\Pi\left(\left\|ZZ^T-Z_0Z_0^T\right\|^2>M\epsilon^2 |X\right)\leq \exp\left(-C'nK_0^2\right)+\exp\left(-C'p+2\log n\right).$$
Note that the first term in the tail dominates, which gives the result of Theorem 4.3. The result of Theorem 4.4 follows a similar argument.
\end{proof}

\subsection{Unknown Variances}

When variances $(\sigma_{A,0}^2,\sigma_{X,0}^2)$ are unknown, we put independent prior $\pi=\pi_A\times \pi_X$ on them, so that
$$([Z],\sigma_A^2,\sigma_X^2)\sim \Pi=\pi_{[Z]}\times\pi_A\times\pi_X,$$
where $\pi_{[Z]}$ is pIBP or IBP on $[Z]$.
In this case, we use the following theorem instead of Theorem 4.1.
\begin{thm} \label{thm:GGV}
Assume
\begin{equation}
\Pi\Big((2\sigma_X^4)^{-1}||\sigma_A^2ZZ^T+\sigma_X^2I-(\sigma_{A,0}^2Z_0Z_0^T+\sigma_{X,0}^2I)||_F^2\leq\epsilon^2\Big)\geq \exp\Big(-Cp\epsilon^2\Big), \label{eq:KLvar}
\end{equation}
for some $\epsilon$ satisfying $p\epsilon^2\rightarrow\infty$ and some constant $C>0$, and there is a testing function $\phi$, such that $E_{Z_0}\phi+\sup_{Z\in U}E_Z(1-\phi)\leq \exp\Big(-(C+4)p\epsilon^2\Big)$, then
$$E_{Z_0}\Pi\Big(U|X\Big)\leq \frac{C'}{p\epsilon^2},$$
for some constant $C'>0$.
\end{thm}
\begin{proof}
In view of Theorem 2.1 of \cite{ghosal00}, we only need to lower bound the prior probability of the Kullback-Leibler neighborhood of the truth. That is, we need to show that (\ref{eq:KLvar}) implies
$$\Pi\left\{E_{Z_0}\left(\log\frac{dP_{Z_0}}{dP_Z}\right)\vee\text{Var}_{P_{Z_0}}\left(\log\frac{dP_{Z_0}}{dP_Z}\right)\leq\epsilon^2\right\}\geq\exp\Big(-Cp\epsilon^2\Big).$$
According to (1) in \cite{chen14}, we have
$$P_Z=N(0,\Sigma)\quad\text{and}\quad P_{Z_0}=N(0,\Sigma_0),$$
where we use the notation $\Sigma=\sigma_{A}^2ZZ^T+\sigma_X^2I$ and $\Sigma_0=\sigma_{A,0}^2Z_0Z_0^T+\sigma_{X,0}^2I$. The same proof of Lemma 8 in \cite{ghosal2007convergence} can be applied to derive the bounds
\begin{equation}
E_{Z_0}\left(\log\frac{dP_{Z_0}}{dP_Z}\right)\vee\text{Var}_{P_{Z_0}}\left(\log\frac{dP_{Z_0}}{dP_Z}\right)\leq \frac{1}{2}\left\|(\Sigma-\Sigma_0)\Sigma^{-1}\right\|_F^2.\label{eq:matrixKL}
\end{equation}
We bound $\frac{1}{2}\left\|(\Sigma-\Sigma_0)\Sigma^{-1}\right\|_F^2$ by
\begin{eqnarray*}
&& \frac{1}{2}\left\|\Sigma-\Sigma_0\right\|_F^2||\Sigma^{-1}||^2 \\
&=&\frac{1}{2}\left\|\Big(\sigma_A^2ZZ^T+\sigma_X^2I-(\sigma_{A,0}^2Z_0Z_0^T+\sigma_{X,0}^2I)\Big)\right\|_F^2\left\|\Big(\sigma_A^2ZZ^T+\sigma_X^2I\Big)^{-1}\right\|^2 \\
&\leq& \frac{1}{2\sigma_X^4}\left\|\sigma_A^2ZZ^T+\sigma_X^2I-(\sigma_{A,0}^2Z_0Z_0^T+\sigma_{X,0}^2I)\right\|_F^2,
\end{eqnarray*}
where the last inequality is because
$$\left\|\Big(\sigma_A^2ZZ^T+\sigma_X^2I\Big)^{-1}\right\|\leq\left(\lambda_{\text{min}}\Big(\sigma_A^2ZZ^T+\sigma_X^2I\Big)\right)^{-1}\leq \sigma_X^{-2}.$$
Therefore, we have
\begin{eqnarray*}
&&\Pi\left\{E_{Z_0}\left(\log\frac{dP_{Z_0}}{dP}\right)\vee\text{Var}_{P_{Z_0}}\left(\log\frac{dP_{Z_0}}{dP}\right)\leq\epsilon^2\right\} \\
&\geq& \Pi\left\{ \frac{1}{2\sigma_X^4}\left\|\sigma_A^2ZZ^T+\sigma_X^2I-(\sigma_{A,0}^2Z_0Z_0^T+\sigma_{X,0}^2I)\right\|_F^2\leq\epsilon^2\right\} \\
&\geq& \exp\Big(-Cp\epsilon^2\Big).
\end{eqnarray*}
Thus, the proof is complete.
\end{proof}

\begin{thm}
Assume $\log p\lesssim n$. Theorem 4.3 and 4.4 still hold if there are universal constants $B>0$ and $C>0$, such that $\sigma_{A,0}^2\in (B^{-1},B)$, $\sigma_{X,0}^2\in (B^{-1},B)$ and $\inf_{t\in (0,2B)}\pi_A(t) \wedge \inf_{t\in (0,2B)}\pi_X(t)\geq CB^{-1}$.
\end{thm}
\begin{proof}
According to Theorem \ref{thm:GGV} and Lemma 4.1, we only need to show
$$ \log\Pi\Big((2\sigma_X^4)^{-1}||\sigma_A^2ZZ^T+\sigma_X^2I-(\sigma_{A,0}^2Z_0Z_0^T+\sigma_{X,0}^2I)||_F^2\leq\epsilon^2\Big) $$
can be lower bounded by the same order of prior mass
in all situations considered  in Section 4.4.  Using conditioning and the independent structure of the prior, we have
\begin{eqnarray*}
&& \Pi\Big((2\sigma_X^4)^{-1}||\sigma_A^2ZZ^T+\sigma_X^2I-(\sigma_{A,0}^2Z_0Z_0^T+\sigma_{X,0}^2I)||_F^2\leq\epsilon^2\Big) \\
&\geq& \Pi\Big((2\sigma_X^4)^{-1}||(\sigma_{A,0}^2-\sigma_A^2)Z_0Z_0^T+(\sigma_{X,0}^2-\sigma_X^2)I||_F^2\leq\epsilon^2\Big)\Pi\Big(||ZZ^T-Z_0Z_0^T||_F^2=0\Big) \\
&\geq& \Pi\Bigg(n^2K_0\left|\frac{\sigma_{A,0}^2-\sigma_A^2}{\sigma_X^2}\right|^2+n\left|\frac{\sigma_{X,0}^2-\sigma_X^2}{\sigma_X^2}\right|^2\leq\epsilon^2\Bigg)\Pi\Big(||ZZ^T-Z_0Z_0^T||_F^2=0\Big),
\end{eqnarray*}
because $||Z_0Z_0^T||_F^2\leq n^2K_0$ and $||I||_F^2=n$. The variance part has lower bound
\begin{eqnarray*}
&& \Pi\Bigg(n^2K_0\left|\frac{\sigma_{A,0}^2-\sigma_A^2}{\sigma_X^2}\right|^2+n\left|\frac{\sigma_{X,0}^2-\sigma_X^2}{\sigma_X^2}\right|^2\leq\epsilon^2\Bigg)\\
&\geq& \Pi\left(n^2K_0\left|\frac{\sigma_{A,0}^2-\sigma_A^2}{\sigma_X^2}\right|^2\leq\epsilon^2/2, n\left|\frac{\sigma_{X,0}^2-\sigma_X^2}{\sigma_X^2}\right|^2\leq\epsilon^2/2\right) \\
&\geq& \Pi\left(n^2K_0B^2\Big(1+\epsilon/\sqrt{2n}\Big)^2|\sigma_A^2-\sigma_{A,0}^2|^2\leq\epsilon^2/2, n\left|\frac{\sigma_{X,0}^2-\sigma_X^2}{\sigma_X^2}\right|^2\leq\epsilon^2/2\right) \\
&=& \pi_A\Bigg(n^2K_0B^2\Big(1+\epsilon/\sqrt{2n}\Big)^2|\sigma_A^2-\sigma_{A,0}^2|^2\leq\epsilon^2/2\Bigg)\pi_X\Bigg(n\left|\frac{\sigma_{X,0}^2-\sigma_X^2}{\sigma_X^2}\right|^2\leq\epsilon^2/2\Bigg).
\end{eqnarray*}
We give lower bounds for the two terms above separately. When $\frac{\epsilon^2}{2n}$ does not go to $0$, $\pi_X\Bigg(n\left|\frac{\sigma_{X,0}^2-\sigma_X^2}{\sigma_X^2}\right|^2\leq\epsilon^2/2\Bigg)$ can be lower bounded by a constant. When it goes to $0$, we have
\begin{eqnarray*}
 \pi_X\Bigg(n\left|\frac{\sigma_{X,0}^2-\sigma_X^2}{\sigma_X^2}\right|^2\leq\epsilon^2/2\Bigg) &\geq& \int_{\frac{\sigma_{X,0}^2\sqrt{2n}}{\sqrt{2n}+\epsilon}}^{\frac{\sigma_{X,0}^2\sqrt{2n}}{\sqrt{2n}-\epsilon}}\pi_X(t)dt \\
 &\geq& C_1B^{-2}\frac{\epsilon}{\sqrt{n}}.
\end{eqnarray*}
Similarly, when $\frac{\epsilon^2}{(1+\epsilon/\sqrt{2n})^2}$ does not go to $0$, $\pi_A\Bigg(n^2K_0B^2\Big(1+\epsilon/\sqrt{2n}\Big)^2|\sigma_A^2-\sigma_{A,0}^2|^2\leq\epsilon^2/2\Bigg)$ can be lower bounded by a constant. When it goes to zero, we have
\begin{eqnarray*}
&& \pi_A\Bigg(n^2K_0B^2\Big(1+\epsilon/\sqrt{2n}\Big)^2|\sigma_A^2-\sigma_{A,0}^2|^2\leq\epsilon^2/2\Bigg) \\
&\geq& \frac{C_2\epsilon}{n\sqrt{K_0}B^2\Big(1+\epsilon/\sqrt{n}\Big)}.
\end{eqnarray*}
To summarize, for any rate $\epsilon$ appearing in Theorems 4.3 and 4.4, we have
$$\Pi\Bigg(n^2K_0\left|\frac{\sigma_{A,0}^2-\sigma_A^2}{\sigma_X^2}\right|^2+n\left|\frac{\sigma_{X,0}^2-\sigma_X^2}{\sigma_X^2}\right|^2\leq\epsilon^2\Bigg)\geq\exp\Big(-C'\big(\log p+\log n+\log K_0\big)\Big),$$
for a constant $C_0$ only depending on $B$. Hence, for Theorem 4.3, we have
\begin{eqnarray*}
&& \Pi\Big((2\sigma_X^4)^{-1}||\sigma_A^2ZZ^T+\sigma_X^2I-(\sigma_{A,0}^2Z_0Z_0^T+\sigma_{X,0}^2I)||_F^2\leq\epsilon^2\Big) \\
&\geq& \exp\Big(-C'\big(\log p+\log n+\log K_0\big)\Big)\times \exp\Big(-CnK_0^2\Big) \\
&\geq& \exp\Big(-C_1nK_0^2\Big),
\end{eqnarray*}
for some $C_1>0$ because $\log p\lesssim n$. Combining this lower bound with Lemma 4.1, the conditions of Theorem \ref{thm:GGV} holds for $\epsilon^2=nK_0^2/p$ and
$$U=\left\{||ZZ^T-Z_0Z_0^T||_F^2>M\frac{K_0^4n^3}{p}\right\},$$
which implies $E_{Z_0}\Pi(U|X)\rightarrow 0$.
For Theorem 4.4, we have
\begin{eqnarray*}
&& \Pi\Big((2\sigma_X^4)^{-1}||\sigma_A^2ZZ^T+\sigma_X^2I-(\sigma_{A,0}^2Z_0Z_0^T+\sigma_{X,0}^2I)||_F^2\leq\epsilon^2\Big) \\
&\geq& \exp\Big(-C'\big(\log p+\log n+\log K_0\big)\Big)\times \exp\Big(-CnK_0^{2(1-\beta)}\Big) \\
&\geq& \exp\Big(-C_2nK_0^{2(1-\beta)}\Big),
\end{eqnarray*}
for some $C_2>0$. Combining this lower bound with Lemma 4.1, the conditions of Theorem \ref{thm:GGV} holds for $\epsilon^2=nK_0^{2(1-\beta)}/p$ and
$$U=\left\{||ZZ^T-Z_0Z_0^T||_F^2>M\frac{K_0^{4-2\beta}n^3}{p}\right\},$$
which implies $E_{Z_0}\Pi(U|X)\rightarrow 0$.
\end{proof}

\subsection{Misspecified Structure}

To handle misspecified structure, we need an argument involving a change of measure. The following bound is a general result for all prior distributions $\Pi$.
\begin{lemma}\label{lem:change}
For any $Z_0\in\{0,1\}^{n\times K_0}$ and $Z^*\in\{0,1\}^{n\times K^*}$,
the following inequality holds for any measurable set $U$,
$$
E_{Z_0}\Pi(U|X) \leq \exp\left(p\fnorm{Z_0Z_0^T-Z^*(Z^*)^T}^2\right)E_{Z^*}\Pi(U|X)+\frac{2}{p\fnorm{Z_0Z_0^T-Z^*(Z^*)^T}^2}.
$$
\end{lemma}
\begin{proof}
Let us use the notation
$$P_{Z_0}=N(0,Z_0Z_0^T+I),\quad\text{and}\quad P_{Z^*}=N(0,Z^*(Z^*)^T+I).$$
By (\ref{eq:matrixKL}) and the bound $\left\|(Z^*(Z^*)^T+I)^{-1}\right\|\leq 1$, we have
$$E_{Z_0}\left(\sum_{j=1}^p\log\frac{dP_{Z_0}}{dP_{Z^*}}(X_j)\right)\vee \text{Var}_{Z_0}\left(\sum_{j=1}^p\log\frac{dP_{Z_0}}{dP_{Z^*}}(X_j)\right)\leq \frac{1}{2}p\fnorm{Z_0Z_0^T-Z^*(Z^*)^T}^2.$$
Define the event
$$B=\left\{\frac{p(X|Z_0)}{p(X|Z^*)}\leq \exp\left(p\fnorm{Z_0Z_0^T-Z^*(Z^*)^T}^2\right)\right\}.$$
By Chebyshev's inequality,
\begin{eqnarray*}
P_{Z_0}(B^c) &=& P_{Z_0}\left\{\log\frac{p(X|Z_0)}{p(X|Z^*)}>p\fnorm{Z_0Z_0^T-Z^*(Z^*)^T}^2\right\} \\
&\leq& P_{Z_0}\left\{\sum_{j=1}^p\left(\log\frac{dP_{Z_0}}{dP_{Z^*}}(X_j)-E_{Z_0}\left(\frac{dP_{Z_0}}{dP_{Z^*}}\right)\right)>\frac{1}{2}p\fnorm{Z_0Z_0^T-Z^*(Z^*)^T}^2\right\} \\
&\leq& \frac{2}{p\fnorm{Z_0Z_0^T-Z^*(Z^*)^T}^2}.
\end{eqnarray*}
Therefore, for any $U$,
\begin{eqnarray*}
E_{Z_0}\Pi(U|X) &\leq& E_{Z_0}\Pi(U|X)\mathbb{I}_B+P_{Z_0}(B^c) \\
&=& E_{Z^*}\frac{p(X|Z_0)}{p(X|Z^*)}\Pi(U|X)\mathbb{I}_B+P_{Z_0}(B^c)\\
&\leq& \exp\left(p\fnorm{Z_0Z_0^T-Z^*(Z^*)^T}^2\right)E_{Z^*}\Pi(U|X)+\frac{2}{p\fnorm{Z_0Z_0^T-Z^*(Z^*)^T}^2}.
\end{eqnarray*}
The proof is complete.
\end{proof}

To apply this result, let us consider a binary factor matrix $Z_0\in \{0,1\}^{n\times K_0}$. It is close to a binary matrix $Z^*\in\{0,1\}^{n\times K_0}$ which has a well-specified group structure with $K_0^*\lesssim K_0^{1-\beta}$. Then, Lemma \ref{lem:change} allows one to bound the posterior probability under the true model $E_{Z_0}\Pi(U|X)$ by $E_{Z^*}\Pi(U|X)$. The object $E_{Z^*}\Pi(U|X)$ can be well bounded because $Z^*$ has an exact two-group structure.

To make this idea work, we need a strengthened version of Theorem 4.4 in the paper with a faster tail probability for certain technical reasons. This can be achieved by the following two lemmas.

\begin{lemma}\label{lem:s1}
For an arbitrary $Z_0\in\{0,1\}^{n\times K_0}$,
under the assumption of Theorem 4.4, there exist some constants $C_1,C_2>0$, such that
$$E_{Z_0}\Pi\left(\left\|ZZ^T+I\right\|_{\infty}\leq C_1(K_0+1)|X\right)\geq 1-\exp\left(-C_2p\right).$$
\end{lemma}
\begin{proof}
We prove the result using the general inequality established in Theorem 4.1 for $U=\{\left\|ZZ^T+I\right\|_{\infty}>C_1K_0\}$. In view of the prior mass lower bound in Theorem 4.2, it is sufficient to establish a test with desired error probability for
$$H_0: Z=Z_0,\quad H_1: \left\|ZZ^T+I\right\|_{\infty}>C_1(K_0+1).$$
Let us decompose the alternative set by
$$H_1\subset \bigcup_{l\geq 1}\left\{C_1l(K_0+1)<||ZZ^T+I||_{\infty}\leq C_1(l+1)(K_0+1)\right\}=\bigcup_{l\geq 1}H_{1l}.$$
Following the proof of Lemma 4.1, there exists $\phi_l$ for each $l\geq 1$, such that
$$E_{Z_0}\phi_l\leq \exp\left(-Clp+2\log n\right),$$
and
$$\sup_{Z\in H_{1l}}E_Z(1-\phi_l)\leq\exp\left(-Cp+2\log n\right).$$
Define $\phi=\max_{l\geq 1}\phi_l$, and then we have
\begin{eqnarray*}
E_{Z_0}\phi+\sup_{Z\in H_1}E_Z(1-\phi) &\leq& E_{Z_0}\phi+\max_{l\geq 1}\sup_{Z\in H_{1l}}E_{Z}(1-\phi) \\
&\leq& \sum_{l\geq 1}E_{Z_0}\phi_l+\max_{l\geq 1}\sup_{Z\in H_{1l}}E_Z(1-\phi_l) \\
&\leq& \sum_{l\geq 1}\exp\left(-Clp+2\log n\right)+\exp\left(-Cp+2\log n\right) \\
&\leq& 2\exp\left(-C'p+2\log n\right).
\end{eqnarray*}
The result follows by applying Theorem 4.1 and the prior mass lower bound in Theorem 4.2.
\end{proof}

\begin{lemma}\label{lem:s2}
Let $Z^*\in\{0,1\}^{n\times K_0}$ be a binary factor matrix with a well specified group structure such that $K_0^*\lesssim K_0^{1-\beta}$ for $\beta\in (0,1)$.
Under the assumption of Theorem 4.4,
$$
E_{Z^*}\Pi\left(\left\|ZZ^T-Z^*(Z^*)^T\right\|_F^2>\eta^2, \left\|ZZ^T+I\right\|_{\infty}\leq C_1(K_0+1)\Big|X\right)
$$
$$\leq  2\exp\left(-Cp\min\Bigg(\frac{\eta^2}{n^2K_0^2},\frac{\eta}{nK_0}\Bigg)+2\log n+C_2nK_0^{2(1-\beta)}\right),$$
for some $C,C_1,C_2>0$.
\end{lemma}
\begin{proof}
We prove this result using Theorem 4.1 for
$$U=\left\{\left\|ZZ^T-Z^*(Z^*)^T\right\|_F^2>\eta^2, \left\|ZZ^T+I\right\|_{\infty}\leq C_1(K_0+1)\right\}.$$
Using the argument in the proof of Lemma 4.1, there is a testing function $\phi$, such that
$$E_{Z^*}\phi+\sup_{Z\in U}E_Z(1-\phi)\leq 2\exp\left\{-Cp\min\Bigg(\frac{\eta^2}{n^2K_0^2},\frac{\eta}{nK_0}\Bigg)+2\log n\right\}.$$
Combining with the prior mass lower bound in Theorem 4.2 and Theorem 4.1, we obtain the result.
\end{proof}

Finally, we are ready to prove Theorem 7.1.
\begin{proof}
Without loss of generality, we assume $\left\|ZZ^T-Z^*(Z^*)^T\right\|_F\geq 1$. The case $\left\|ZZ^T-Z^*(Z^*)^T\right\|_F<1$ implies that $\left\|ZZ^T-Z^*(Z^*)^T\right\|_F=0$ and has been treated by Theorem 4.4.
Define
$$V=\left\{\left\|ZZ^T-Z^*(Z^*)^T\right\|_F^2>\eta^2\right\},$$
for some $\eta$ to be specified later.
First, we use union bound to obtain
\begin{eqnarray*}
E_{Z_0}\Pi(V|X) &\leq& E_{Z_0}\Pi\left(V, \left\|ZZ^T+I\right\|_{\infty}\leq C_1(K_0+1)|X\right) \\
&&+E_{Z_0}\Pi\left(\left\|ZZ^T+I\right\|_{\infty}> C_1(K_0+1)|X\right),
\end{eqnarray*}
where the second term is bounded by $\exp(-C_2p)$ according to Lemma \ref{lem:s1}. For the first term, we bound it by
\begin{eqnarray*}
&& E_{Z_0}\Pi\left(V, \left\|ZZ^T+I\right\|_{\infty}\leq C_1(K_0+1)|X\right) \\
&\leq& \exp\left(p\fnorm{Z_0Z_0^T-Z^*(Z^*)^T}^2\right)E_{Z^*}\Pi\left(V, \left\|ZZ^T+I\right\|_{\infty}\leq C_1(K_0+1)|X\right)\\
&& +\frac{2}{p\fnorm{Z_0Z_0^T-Z^*(Z^*)^T}^2} \\
&\leq& 2\exp\left(-Cp\min\Bigg(\frac{\eta^2}{n^2K_0^2},\frac{\eta}{nK_0}\Bigg)+2\log n+C_2nK_0^{2(1-\beta)}+p\fnorm{Z_0Z_0^T-Z^*(Z^*)^T}^2\right)\\
&& +\frac{2}{p},
\end{eqnarray*}
where the first inequality is due to Lemma \ref{lem:change}, and the second inequality is due to Lemma \ref{lem:s2} and $\fnorm{Z_0Z_0^T-Z^*(Z^*)^T}^2\geq 1$.
Choosing
$$\eta^2 =M'\frac{n^4K_0^{6-4\beta}}{p^2}+n^2K_0^2\fnorm{Z_0Z_0^T-Z^*(Z^*)^T}^4,$$
for some sufficiently large $M'>0$, we have
$$p\min\Bigg(\frac{\eta^2}{n^2K_0^2},\frac{\eta}{nK_0}\Bigg)\asymp nK_0^{2(1-\beta)}+p\fnorm{Z_0Z_0^T-Z^*(Z^*)^T}^2.$$
Then,
$$E_{Z_0}\Pi\left(V|X\right)\leq \exp\left(-C_1nK_0^{2(1-\beta)}\right)+\exp\left(-C_2p\right)+\frac{2}{p}\leq \frac{C_3}{p}\leq \exp\left(-C'nK_0^{2(1-\beta)}\right)+\frac{2}{p}.$$
Finally, observe that
\begin{eqnarray*}
V &\supset& \left\{\left\|ZZ^T-Z_0(Z_0)^T\right\|_F^2\geq M_1\left(\frac{n^4K_0^{6-4\beta}}{p^2}+n^2K_0^2\fnorm{Z_0Z_0^T-Z^*(Z^*)^T}^4\right)\right\} \\
&\supset& \left\{\left\|ZZ^T-Z_0(Z_0)^T\right\|_F^2\geq M\left(\frac{n^3K_0^{4-2\beta}}{p}+n^2K_0^2\fnorm{Z_0Z_0^T-Z^*(Z^*)^T}^4\right)\right\}
\end{eqnarray*}
for some $M>0$, where the last inequality is because $\frac{n^3K_0^{4-2\beta}}{p}\gtrsim \frac{n^4K_0^{6-4\beta}}{p^2}$ under the assumption of Theorem 4.4.
Hence,
we obtain the desired posterior contraction for $\left\|ZZ^T-Z_0(Z_0)^T\right\|_F^2$.
\end{proof}

\begin{proof}[Proof of Corollary 7.1]
It is sufficient to bound $||Z_0Z_0^T-Z^*(Z^*)^T||_F^2$. By triangle inequality, we have
\begin{eqnarray*}
||Z_0Z_0^T-Z^*(Z^*)^T||_F^2 &\leq& \left(||Z_0(Z_0-Z^*)^T||_F + ||(Z_0-Z^*)(Z^*)^T||_F\right)^2 \\
&\leq& (||Z_0||+||Z^*||)^2||Z_0-Z^*||^2_F.
\end{eqnarray*}
Note that $Z^*$ is obtained by zeroing out entries in $Z_0$, and thus we have $||Z^*||\leq ||Z_0||$. Since there are at most $O(n^{\delta})$ entries being zeroed out, we have $||Z_0-Z^*||^2_F\lesssim n^{\delta}$. To summarize, we obtain the bound $||Z_0Z_0^T-Z^*(Z^*)^T||_F^2\lesssim n^{\delta}||Z_0||^2$. The requirement that $(nK_0)^2n^{2\delta}||Z_0||^4=o(K_0^4n^3/p)$ leads to the condition $n^{2\delta}=o\left(\frac{nK_0^{2}}{p||Z_0||^4}\right)$. Thus, the proof is complete.
\end{proof}

\section{Proof of Technical Lemmas} \label{sec:tech}

To prove Lemma \ref{lem:testing}, we need the following large deviation inequality.
\begin{lemma}\label{lem:bernstein}
For $\{W_{i1},W_{i2}\}_{i=1}^p$ from i.i.d. bi-variate normal distribution with $\text{Var}(W_{i1})=\text{Var}(W_{i2})=1$ and $\text{Cov}(W_{i1},W_{i2})=\rho$, we have for any $\epsilon>0$,
$$P\left\{\left|\frac{1}{p}\sum_{i=1}^p(W_{i1}W_{i2}-E(W_{i1}W_{i2}))\right|>\epsilon\right\}\leq \exp\big(-Cp(\epsilon\wedge \epsilon^2)\big),$$
 for some $C>0$.
\end{lemma} 
\begin{proof}
Since $W_{i1}$ and $W_{i2}$ are from normal distribution, $W_{i1}W_{i2}$ is a sub-exponential random variable. To be specific, let us consider the case $\rho\geq 0$ without loss of generality. Then, $W_{i1}$ and $W_{i2}$ can be represented as
$$W_{i1}=\sqrt{\rho}Z+\sqrt{1-\rho}U,\quad W_{i2}=\sqrt{\rho}Z+\sqrt{1-\rho}V,$$
with $U,V,Z$ i.i.d. $N(0,1)$. Then,
\begin{eqnarray*}
&& P\left\{|W_{i1}W_{i2}-\rho|>t\right\} \\
 &=& P\left\{\left|\rho(Z^2-1)+\sqrt{\rho(1-\rho)}(ZU+ZV)+(1-\rho)UV\right|>t\right\} \\
&\leq& P\left\{|\rho(Z^2-1)|>\frac{t}{3}\right\} + P\left\{|\sqrt{\rho(1-\rho)}(ZU+ZV)|>\frac{t}{3}\right\}+P\left\{|(1-\rho)UV|>\frac{t}{3}\right\} \\
&\leq& P\left\{|Z^2-1|>\frac{t}{3}\right\} + P\left\{|Z(U+V)|>\frac{t}{3}\right\}+P\left\{|UV|>\frac{t}{3}\right\} \\
&\leq& \exp(-Ct),
\end{eqnarray*}
for some constant $C>0$. The last inequality above holds because $|Z^2-1|$, $|Z(U+V)|$ and $|UV|$ all have bounded sub-exponential norm. We have shown that $|W_{i1}W_{i2}-\rho|$ has bounded sub-exponential norm. For the case when $\rho<0$, we can represent $W_{i2}$ by $-\sqrt{\rho}Z-\sqrt{1-\rho}V$.
By Proposition 5.16 of \cite{vershynin10}, the conclusion follows.
\end{proof}

\begin{proof}[Proof of Lemma \ref{lem:testing}]
Let $\frac{1}{p}XX^T=(\hat{\sigma}_{st})_{n\times n}$ and $ZZ^T+I=(\sigma_{st})_{n\times n}$. Then we have
\begin{eqnarray*}
&& P_{Z}\left\{\left\|\frac{1}{p}XX^T-(ZZ^T+I)\right\|_F>\epsilon\right\} = P_Z\left\{\sum_{s,t}(\hat{\sigma}_{st}-\sigma_{st})^2>\epsilon^2\right\}\\
&\leq& \sum_{s,t}P_{Z}\left\{(\hat{\sigma}_{st}-\sigma_{st})^2>\frac{\epsilon^2}{n^2}\right\} \leq \sum_{s,t}P_{Z}\left\{\frac{(\hat{\sigma}_{st}-\sigma_{st})^2}{\sigma_{ss}\sigma_{tt}}>\frac{\epsilon^2}{n^2||ZZ^T+I||_{\infty}^2}\right\}.
\end{eqnarray*}
Using Lemma \ref{lem:bernstein}, the above quantity can be upper bounded by
\begin{eqnarray*}
&& \sum_{s,t}\exp\left\{-Cp\min\Bigg(\frac{\epsilon^2}{n^2||ZZ^T+I||_{\infty}^2},\frac{\epsilon}{n||ZZ^T+I||_{\infty}}\Bigg)\right\} \\
&=& \exp\left\{-Cp\min\Bigg(\frac{\epsilon^2}{n^2||ZZ^T+I||_{\infty}^2},\frac{\epsilon}{n||ZZ^T+I||_{\infty}}\Bigg)+2\log n\right\}.
\end{eqnarray*}
This proves the first inequality.
Using the same argument, we have
\begin{eqnarray*}
&& P_{Z}\left\{\left\|\frac{1}{p}XX^T-(ZZ^T+I)\right\|_{\infty}>\epsilon\right\} \leq \sum_{s,t}P_{Z}\left\{|\hat{\sigma}_{st}-\sigma_{st}|>\epsilon\right\} \\
&\leq& \sum_{s,t}P_Z\left\{\frac{(\hat{\sigma}_{st}-\sigma_{st})^2}{\sigma_{ss}\sigma_{tt}}>\frac{\epsilon^2}{||ZZ^T+I)||^2_{\infty}}\right\} \\
&\leq& \exp\left\{-Cp\min\Bigg(\frac{\epsilon^2}{||ZZ^T+I)||_{\infty}^2},\frac{\epsilon}{||ZZ^T+I)||_{\infty}}\Bigg)+2\log n\right\},
\end{eqnarray*}
which proves the second inequality.
\end{proof}

\section{Date analysis using alternative methods} \label{sec:apply}

To compare with the real data analysis in \cite{chen14} using a pIBP prior, we analyzed the same 134 breast cancer samples with the expression profiles of 300 genes and the mutation status of 11 genes with IBP prior. The resulting latent factor matrix is less sparse than that of pIBP, which offers compromised interpretability. Moreover, the  reported features in \cite{chen14} were not recovered by IBP prior, suggesting the integration of somatic mutations might lead to better understanding of gene expression (Supplementary Figure \ref{fig1}). 

\begin{figure}
\centering
\includegraphics[width=5in]{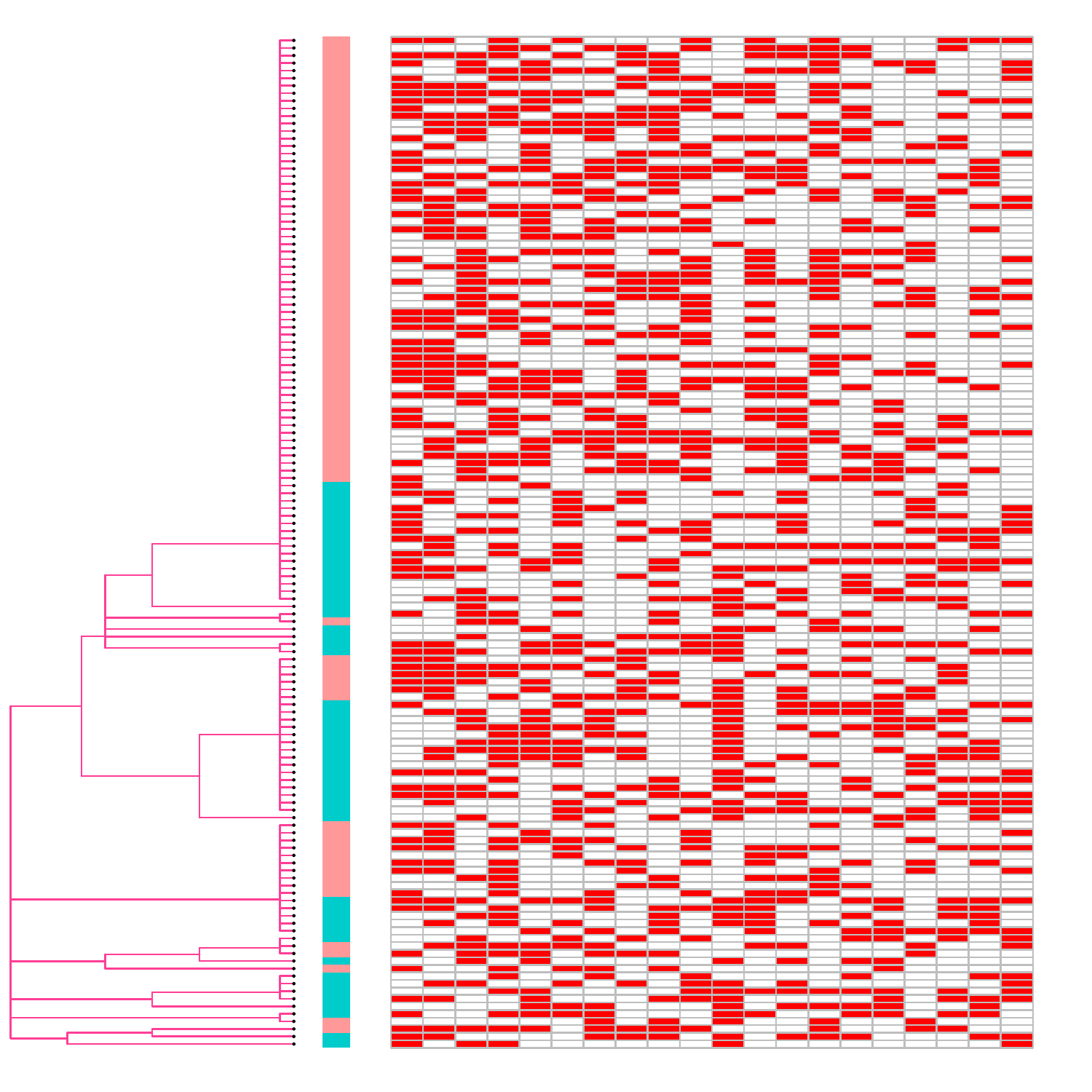}
\caption{IBP result on TCGA breast cancer samples. This plot shows the dendrogram tree prior (left), the inferred latent factor matrix Z (right, only first 20 columns shown) and subtype status (middle, Basal-like as Red and HER2 as Green). }
\label{fig1}
\end{figure}


\end{document}